\documentclass{applemlr}

\usepackage{amsmath}
\usepackage{enumerate}
\usepackage{algorithm}
\usepackage{algpseudocode}
\usepackage{amsfonts}
\usepackage{amsthm}
\usepackage{cleveref}
\usepackage{diagbox}
\usepackage{colortbl}
\usepackage{amssymb}
\usepackage{xspace}
\usepackage{wrapfig}
\usepackage{adjustbox}
\usepackage{tabularx}
\usepackage{booktabs}
\usepackage{mathtools}
\usepackage{tikz}
\usepackage{enumitem}
\usepackage{silence}
\usepackage{dsfont}
\usepackage[table]{xcolor}
\usepackage[dvipsnames]{xcolor}
\usepackage{multirow}
\usepackage{makecell}
\usepackage{xfakebold}

\usepackage{amsmath,amsfonts,bm}

\def\eqref#1{equation~\ref{#1}}

\def\1{\bm{1}}

\DeclareMathAlphabet{\mathsfit}{\encodingdefault}{\sfdefault}{m}{sl}
\SetMathAlphabet{\mathsfit}{bold}{\encodingdefault}{\sfdefault}{bx}{n}

\definecolor{textgray}{HTML}{6E6E73}
\usetikzlibrary{positioning, calc}
\usetikzlibrary{decorations.pathmorphing}

\makeatletter
\patchcmd{\wrong@fontshape}{\@gobbletwo}{}{}{}
\makeatother
\numberwithin{equation}{section}
\makeatletter
\AtBeginDocument{
  \urlstyle{sf}
  
}
\makeatother

\definecolor{light}{RGB}{125, 125, 125}
\crefname{tcb@cnt@pbox}{code}{code}
\Crefname{tcb@cnt@pbox}{Code}{Code}
\crefname{assumption}{assumption}{assumption}
\Crefname{assumption}{Assumption}{Assumptions}

\newtcolorbox[auto counter]{pbox}[2][]{
  colback=white,
  title=Code~\thetcbcounter: #2,
  #1,fonttitle=\sffamily,
  fontupper=\sffamily,
  arc=2pt,
  colframe=bgcolor,
  coltitle=fgcolor,
  colbacktitle=bgcolor,
  toptitle=0.25cm,
  bottomtitle=0.125cm
}

\makeatletter
\newcommand\applefootnote[1]{%
  \begingroup
  \renewcommand\thefootnote{}%
  \renewcommand\@makefntext[1]{\noindent##1}%
  \footnote{#1}%
  \addtocounter{footnote}{-1}%
  \endgroup
}
\makeatother

\definecolor{cverbbg}{gray}{0.90}

\usepackage{graphicx}
\usepackage{subcaption}
\usepackage{nicefrac}
\usepackage{microtype}
\usepackage{float}
\usepackage{pgfplots}
\pgfplotsset{compat=1.18}
\usetikzlibrary{decorations.pathreplacing}

\definecolor{cellgreen}{RGB}{180,230,175}
\definecolor{cellblue}{RGB}{175,208,240}
\definecolor{cellgray}{RGB}{232,232,232}
\definecolor{cellyellow}{RGB}{255,248,175}
\definecolor{cellred}{RGB}{255,195,195}
\definecolor{cellpurple}{RGB}{220,200,240}

\definecolor{barVanilla}{RGB}{155,155,155}
\definecolor{barTDA}{RGB}{120,170,225}
\definecolor{barSheaf}{RGB}{38,100,190}

\definecolor{cgn4wh}{rgb}{0.9600,0.9800,0.9600}
\definecolor{cgr9wh}{rgb}{0.9550,0.9550,0.9550}
\definecolor{cbl6wh}{rgb}{0.9400,0.9400,1.0000}

\newcolumntype{L}[1]{>{\raggedright\arraybackslash}p{#1}}
\newcolumntype{K}[1]{>{\centering\arraybackslash}p{#1}}

\newtheorem{theorem}{Theorem}[section]

\providecommand{\best}[1]{\textbf{#1}}

\providecommand{\dneg}[1]{\textbf{\ensuremath{-}#1}}

\title{TopoPrimer: The Missing Topological Context in Forecasting Models}

\author{Zara Zetlin}
\author{Kayhan Moharreri}
\author{Maria Safi}

\affiliation{Apple}

\abstract{
We introduce \textbf{TopoPrimer}, a framework that makes the global topological structure
  of the series population an explicit input to any forecasting
  model. TopoPrimer improves accuracy across diverse domains, stabilizes forecasts under seasonal demand spikes, and closes the cold-start gap. Precomputed once per domain via persistent homology and spectral sheaf coordinates, TopoPrimer deploys per token for fully-trained models and as a lightweight adapter for pre-trained backbones. Of these two components, sheaf coordinates are the primary accuracy driver. Across four public benchmarks on Chronos and TimesFM, TopoPrimer consistently improves forecasting accuracy, with gains of up to 7.3\% MSE on ECL. The topology advantage persists with near-identical magnitude across
  zero-shot and fine-tuned backbones, suggesting topology and
  per-series training capture complementary signals.
  The gains are most pronounced in difficult regimes.
  Under peak seasonal demand, classical and zero-shot models degrade by up to 50\%, while TopoPrimer  stays within 10\%. At cold start with no item
  history, TopoPrimer reduces MAE by 27\% over a topology-free
  baseline.
}

\metadata[Correspondence]{\sffamily Zara Zetlin: \url{zzetlinfishbein@apple.com}}
\date{\sffamily\today}

\begin{document}

\maketitle

\section{Introduction}
{\looseness=-1 Time series foundation models (TSFMs) such as 
Chronos~\citep{ansari2025chronos2} and TimesFM~\citep{das2024timesfm} 
have fundamentally shifted the forecasting paradigm. Pre-trained on 
billions of series from diverse corpora, they generalize across 
domains without per-dataset fine-tuning. Each series is encoded from 
its own token history, and cross-series reasoning is learned only 
implicitly through attention. This architecture is powerful, yet it 
leaves one source of information unexploited: the global topological 
structure of the series~population.}

In any real-world forecasting domain, whether energy grids, retail supply chains, or road traffic networks, the full collection of series forms a manifold with coherent, informative geometry. Within this manifold, series can be grouped behaviorally, form loops of
  co-movement, and be naturally divided into distinct regions. Crucially, this structure cannot be observed from any individual series alone. Yet, across the series population, it constitutes a systematic, recoverable signal which could be used at every forecast step.

To capture this signal, we introduce
  TopoPrimer, a framework that encodes the topological shape and
  relational population structure as a frozen
  precomputed input to any forecasting backbone. To create TopoPrimer's topological context vector, we apply two tools grounded in algebraic topology. The first is topological data analysis (TDA), specifically persistent
  homology,
  to capture topological shape across scales. While prior forecasting work applies
  persistent homology to sliding-window embeddings of individual
series \citep{zeng2021topological,lin2025crosstoponet,lin2025tis},
  we instead apply it to the cross-series correlation manifold
(Figure~\ref{fig:population_tda}). This produces a 125-dimensional persistence
   landscape fingerprint encoding global clustering ($H_0$), cyclic
  co-movement ($H_1$), and boundary structure ($H_2$), computed once
  per domain and shared across all series.

The second tool we use is cellular sheaf theory~\citep{curry2014sheaves,hansen2021sheaf}, which  
  describes how each series is situated within 
  the full domain. Prior sheaf work computes   
  this via learned graph convolutions,  

       \newpage
  replacing or augmenting the backbone
entirely~\citep{li2018dcrnn,wu2019graphwavenet,bodnar2022neural,mostafa2026stsheaf}.   We instead derive the sheaf coordinate without learned graph convolutions, keeping the
  topology signal backbone-agnostic. Rather than training a full sheaf network, 
  we initialize this embedding spectrally via truncated SVD of the entity-time matrix and find the closed-form result superior to the trained
   alternative.
This produces a 256-dimensional spectral representation per series encoding relational 
  position and cross-entity similarity, computed once per domain and unique to each series.

Each of these topology components is projected to a common hidden dimension. In the 
  fully-trained setting, these projections are summed into a single context vector  
  that is broadcast-added to every temporal input token. In the pre-trained setting,
   a lightweight adapter merges the topology projections with the frozen base
  forecast to apply topology-informed residual corrections. The adapter is less than 0.1\% of either Chronos or TimesFM, and trains
  entirely on cached base forecasts with no gradient through the backbone.

TopoPrimer consistently improves accuracy across diverse domains, limits degradation under seasonal demand spikes, and closes the cold-start gap. Across four public datasets, MAE falls by 7.9\% on Monash Weather with the fully-trained
   Transformer. In the pre-trained setting, MSE falls by 7.3\% with Chronos and 6.8\% with TimesFM
  on ECL. Notably, the topology advantage persists on a fine-tuned backbone, suggesting population-level topological structure captures a complementary signal. These gains are most pronounced in difficult regimes. Under peak seasonal demand,
  TopoPrimer degrades by under 10\%, while classical models
  and zero-shot TSFMs such as Chronos degrade by up to 50\%. At cold start, where no item history exists at launch, TopoPrimer reduces MAE by
27\% over a vanilla topology-free Transformer.  These results demonstrate cross-series topology as a useful forecasting signal, 
  injectable into any model at minimal~cost.

\paragraph{Contributions.} We make the following contributions:
\begin{itemize}
\item \textbf{Population-level TDA as a forecasting feature.} We apply persistent homology to the cross-series correlation manifold rather than to individual series, producing a shared persistence landscape vector that encodes global clustering, cyclic co-movement, and boundary structure across the full domain. To our knowledge, this is the first application of TDA to the population manifold for forecasting.

{\looseness=-1\item \textbf{Spectral sheaf coordinates as a per-series  relational prior.} We derive the spectral
  form of this coordinate directly from the leading left singular vectors of the entity-time
  matrix, requiring no training or graph construction. Grounded in cellular sheaf theory, these coordinates capture each series' position and relational structure within the full population, encoding where a series sits relative to dominant patterns across the domain.}
  
{\looseness=-1\item \textbf{A unified framework across training paradigms.} The same topology features improve both fully-trained transformers and frozen pre-trained TSFMs under a single architecture, demonstrating how population topology is a broadly useful signal  across backbone families. }
\end{itemize}

\begin{figure}[t]
  \centering
  \includegraphics[width=0.88\linewidth]{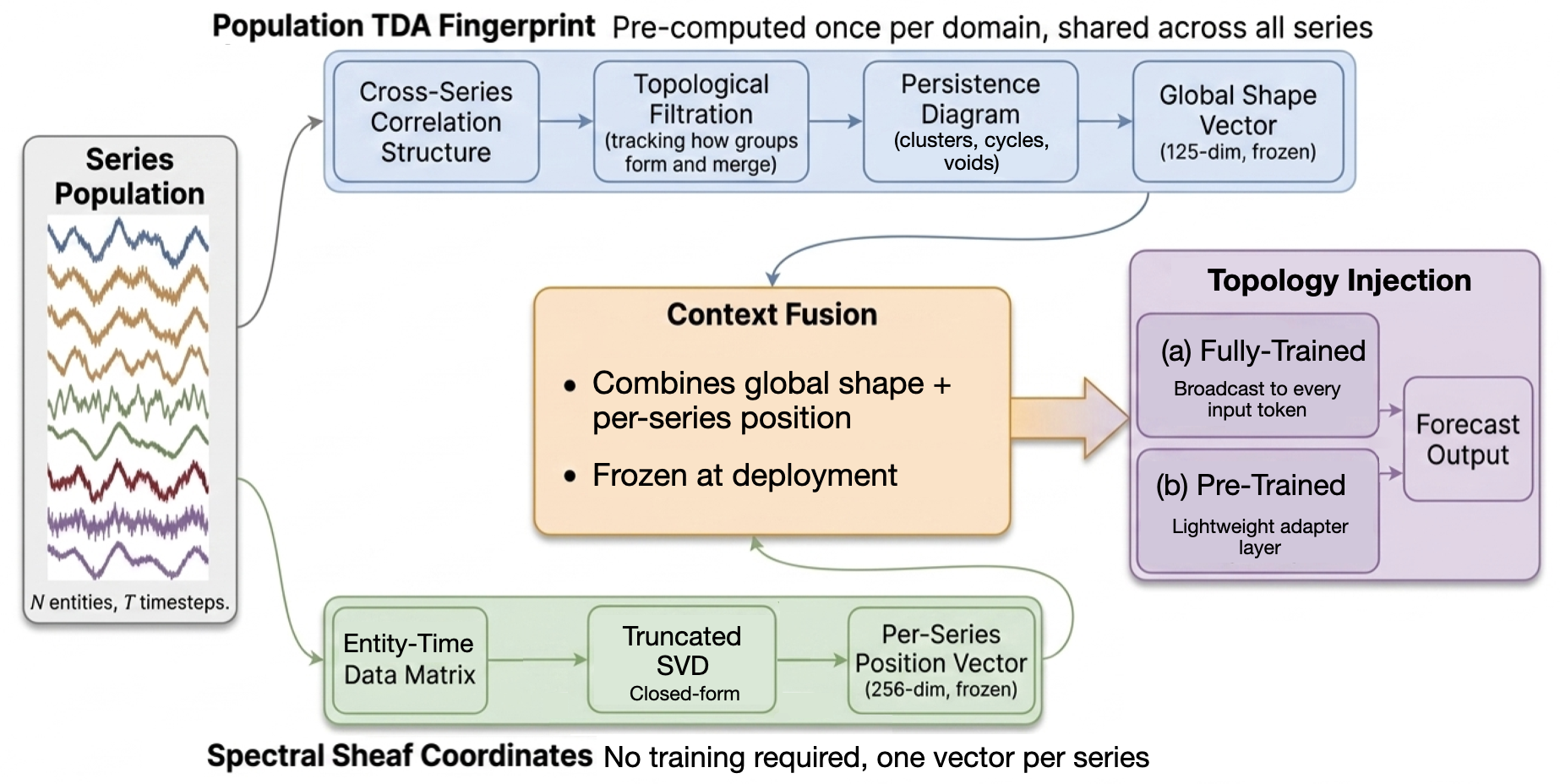}
  \caption{%
    \textbf{TopoPrimer architecture overview.}
    Two frozen signals are extracted offline from the series population: a                  
    125-dimensional global TDA fingerprint via topological filtration
    (top), and a 256-dimensional per-series spectral sheaf coordinate via                   
    truncated SVD (bottom). After fusion, the combined context is injected into
    the backbone either by broadcast-addition to every input token
    \textit{(a) fully-trained} or via a lightweight adapter with the backbone
    frozen \textit{(b) pre-trained}. Backbone weights are
    never modified; both components require no
    gradient.
  }
  \label{fig:architecture}
\end{figure}

\section{Related Work}
\label{sec:related}
\paragraph{Topological deep learning and TDA for time series.} Topological deep 
  learning (TDL)~\citep{papillon2024tdl} shapes neural network architecture around 
  the topology of the underlying data space. Within forecasting, prior work applies 
  persistent homology~\citep{carlsson2009topology,edelsbrunner2010computational} to 
  sliding-window embeddings of individual   series~\citep{zeng2021topological,lin2025crosstoponet,lin2025tis,kim2025topocl}.
  These methods capture within-series temporal dynamics, such as periodicity and
  local shape, but each window produces its own descriptor. The geometry of the
  broader population is never modeled. Instead, TopoPrimer applies persistent
  homology directly to the cross-series correlation manifold, producing one shared
  fingerprint for the entire domain. This reframing, from per-series temporal
  topology to population-level relational topology, is the core methodological
  departure from prior TDA forecasting work.

  \newpage 
\paragraph{Graph and relational forecasting.}
Graph-based forecasters such as DCRNN~\citep{li2018dcrnn}, Graph
WaveNet \citep{wu2019graphwavenet}, and MTGNN~\citep{wu2020mtgnn} learn directed or adaptive adjacency over fixed entity graphs,
  replacing or augmenting the backbone for each domain.
Transformer-based models~\citep{zhou2021informer,lim2021tft,nie2023patchtst}
sidestep relational structure entirely, encoding each series independently.
Most similar to ours, global-factor models~\citep{wang2019deepfactors} learn a low-rank factorization jointly with the forecast objective, producing latent per-series coordinates, but as learned embeddings rather than a closed-form frozen prior.
 Unlike all of these, TopoPrimer does not
replace or modify the backbone; it injects population topology as a
precomputed context that any existing model can consume without modification.

\paragraph{Cellular sheaf methods.}
Cellular sheaf theory~\citep{curry2014sheaves,hansen2021sheaf} extends graph
convolution by assigning restriction maps to node-edge incidences, enabling
relational structure that shared-weight message-passing cannot represent.
Bodnar et al.~\citep{bodnar2022neural} learn distinct per-incidence
restriction maps on heterophilic graphs; ST-Sheaf
GNN~\citep{mostafa2026stsheaf} applies diagonal maps for spatio-temporal
forecasting, using the sheaf network itself as the full model. Both remain
locally focused: each node's representation is shaped by its immediate
neighbors with no view of its position within the broader population.
TopoPrimer instead derives each series' coordinate from the leading left singular vectors of 
  the entity-time matrix in closed form, requiring no training. Deriving spectral sheaf 
  coordinates as a frozen, backbone-agnostic prior for time series forecasting is an approach    
  that prior sheaf work has not, to our knowledge, explored.

\paragraph{Time series foundation models.}
TSFMs such as
  Chronos~\citep{ansari2025chronos2} and TimesFM~\citep{das2024timesfm}  are designed for zero-shot transfer across domains. When adaptation is needed, the model is updated via fine-tuning on individual series histories.
  Neither regime introduces explicit population-topology signals. TopoPrimer does, by injecting  precomputed population-level TDA features and per-series spectral sheaf coordinates as a frozen, backbone-agnostic context vector.

\newpage 
\section{Method}
\label{sec:method}
TopoPrimer treats topology as a precomputed prior, not a learned component. Two signals are extracted offline
   once per domain, a population TDA fingerprint and per-series spectral sheaf coordinates. These are
  fused into a context vector, and injected into any forecasting backbone without weight
  modification (Figure~\ref{fig:architecture}). We describe each signal in turn, then detail
  injection for the fully-trained and pre-trained settings. Mathematical definitions appear in
  Appendix~\ref{app:prelim}.

\subsection{Population TDA Fingerprint}
\label{sec:tda}
\paragraph{Correlation manifold.}
Given $N$ series, we form an $N \times T$ matrix $\mathbf{X}$ of normalized
historical observations and compute the correlation-distance matrix
$D_{ij} = 1 - |\rho_{ij}|$, where $\rho_{ij}$ is the Pearson correlation
between series $i$ and $j$. For large populations we sparsify via $k=50$
nearest neighbors, since it reduces memory from $O(N^2)$ to $O(Nk)$, sufficient for the population
  sizes in our domains.
We then apply
persistent homology to this manifold. The
resulting persistence landscape is Lipschitz-continuous with respect to the
data distribution (Appendix~\ref{app:lipschitz}), so the fingerprint degrades
gracefully under noise.

\paragraph{Vietoris-Rips filtration.}
We run a Vietoris-Rips filtration~\citep{ctralie2018ripser} up to dimension 2, covering the three fundamental topological primitives. Higher dimensions are 
  computationally expensive and empirically absent in correlation manifolds of      
  typical scale.
We extract $H_0$ (clustering), $H_1$ (cyclic co-movement), and $H_2$ (structural boundary)
features as birth-death pairs across the filtration. Long-lived features represent
robust population structure and short-lived ones are noise. Formal definitions
appear in Appendix~\ref{app:prelim}.

\paragraph{Persistence landscape vectorization.}
We convert each persistence diagram to a fixed-size vector via the persistence
  landscape~\citep{bubenik2015statistical} (definition in
  Appendix~\ref{app:prelim}). We sample landscape layers $\lambda_1$ and $\lambda_2$
   at 25 points each for $H_0$ and $H_1$, and $\lambda_1$ only at 25 points for
  $H_2$, where voids are sparse and $\lambda_2$ contributes noise rather than
  signal. Including $\lambda_2$ for $H_0$ and $H_1$ captures secondary structure,
  such as a two-cluster market split, that the top landscape alone misses. This
  yields a \textbf{125-dimensional TDA fingerprint} ($50 + 50 + 25$), computed once
  per domain and broadcast identically to all series.

\subsection{Sheaf Encoder}
\label{sec:sheaf}
\paragraph{Spectral sheaf coordinates.}
  While the TDA fingerprint captures the global shape of the series population, 
  the sheaf component provides a complementary per-series signal, encoding where each
   series sits relative to others in the domain. A cellular 
sheaf~\citep{curry2014sheaves,hansen2021sheaf} assigns a spectral coordinate to   
  each series based on its relational position within the population; the formal 
  derivation appears in Appendix~\ref{app:prelim}.

Concretely, this coordinate is row~$i$ of $U$, the left factor of a truncated
  singular value decomposition (SVD) $\mathbf{X} \approx U \Sigma V^\top$, where
  $\mathbf{X} \in \mathbb{R}^{N \times T}$ is the entity-time matrix over the full
  dataset (Figure~\ref{fig:sheaf_spectral}). When series span unrelated categories, as in M5, where 30,490 item-store
   series cross category boundaries, we partition into semantically coherent groups
   and apply SVD within each. The resulting coordinate retains all available singular vectors and is
  zero-padded to 256 dimensions, giving the \textbf{spectral relational feature} of
  series~$i$.

We evaluate a learned neural sheaf encoder as an alternative in
  Appendix~\ref{app:spectral_vs_neural}. Spectral coordinates uniformly outperform
  the neural sheaf encoder at a fraction of the cost, and are adopted as default. The TDA fingerprint is
  \emph{global} (one shared vector per domain), whereas spectral relational features are
  \emph{per-series} (each series' coordinate in $U$ locates it within the shared
  demand manifold).

\begin{figure}[h]
  \centering
  \includegraphics[width=0.65\textwidth]{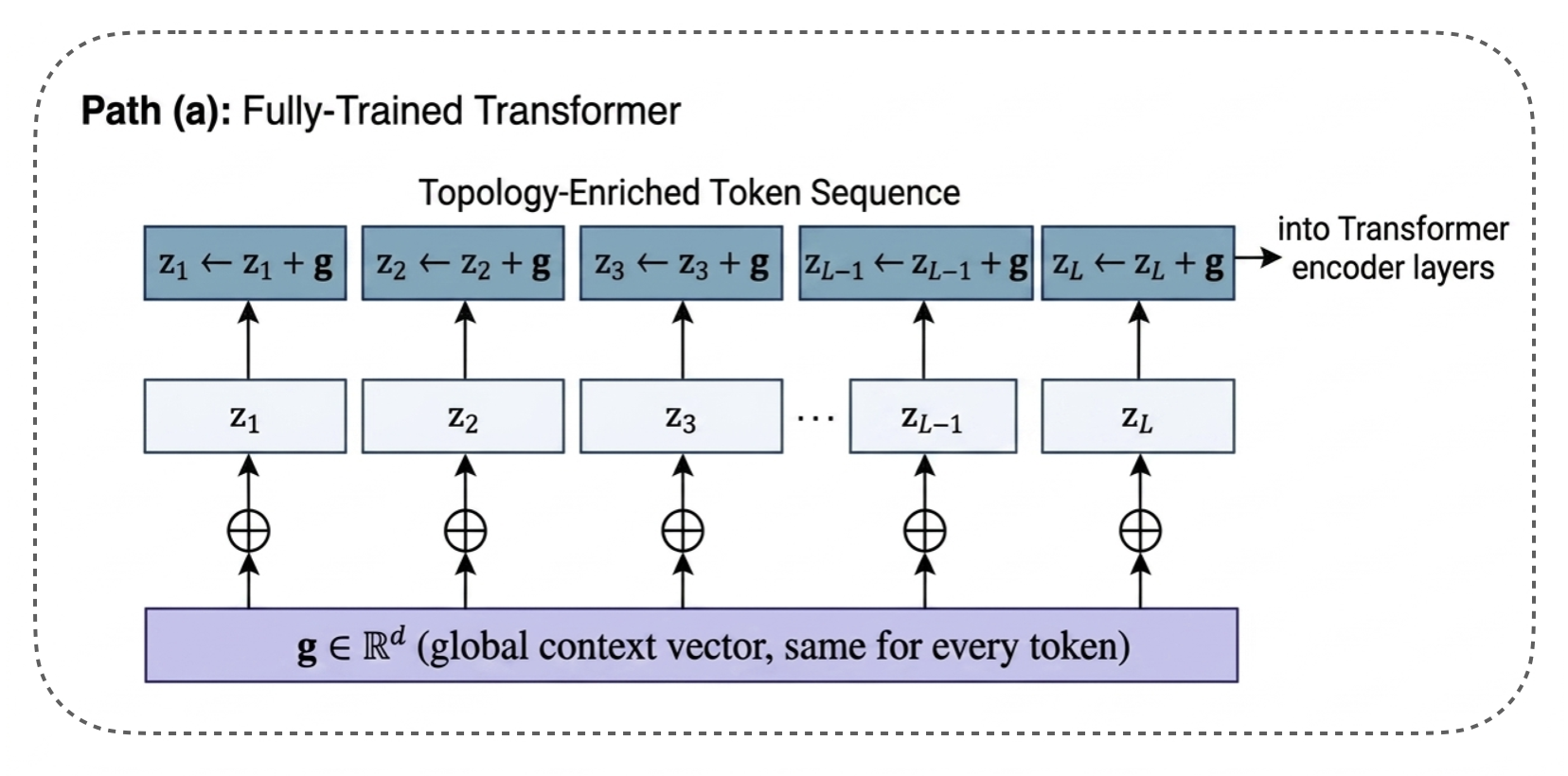}
  \caption{%
    \textbf{Global context broadcast injection (Path~(a)).}
    $\mathbf{g} \in \mathbb{R}^d$ is broadcast-added to every temporal token
    before the encoder. 
    No gradient flows through $\mathbf{g}$.%
  }
  \label{fig:broadcast_injection}
\end{figure}

\begin{figure}[!ht]
  \centering
  \includegraphics[width=0.99\linewidth]{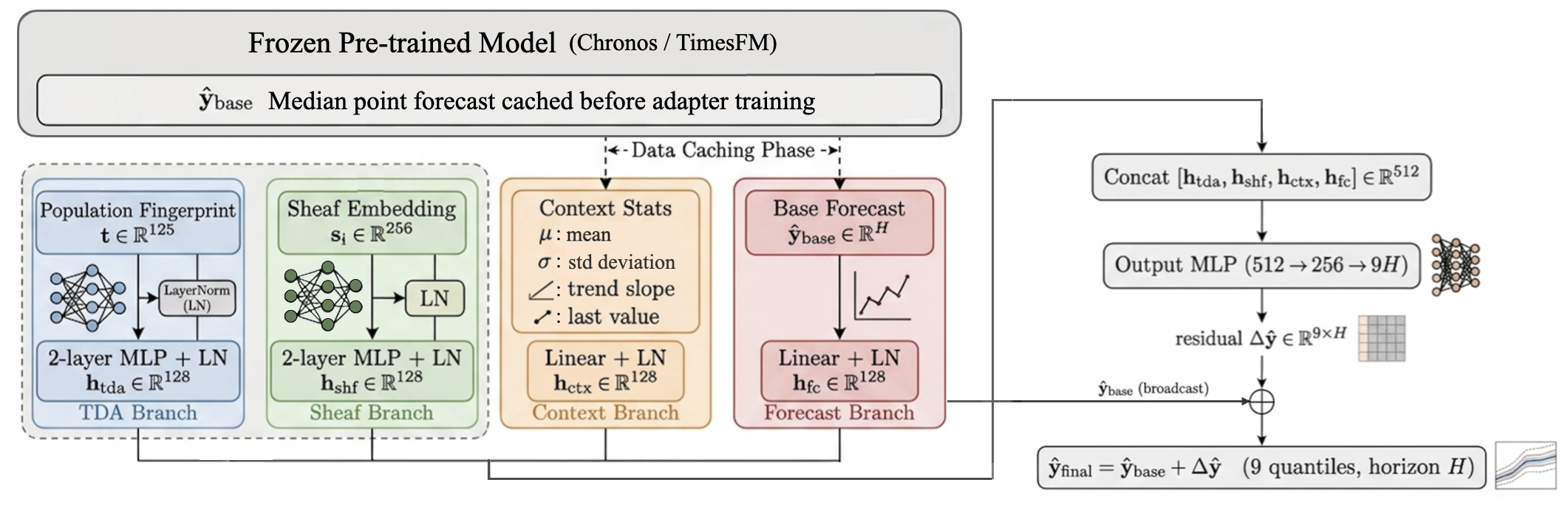}
  \caption{%
    \textbf{Topology adapter for frozen pre-trained backbones.}
    Four independent branches each project to a common hidden dimension
    $H{=}128$: the TDA fingerprint (blue), the per-series spectral sheaf
    coordinate (green), four z-scored context statistics (orange), and the frozen
    backbone's cached median forecast (red).
    Separate projections ensure no branch dominates
  by sheer input dimensionality.
    The concatenated 512-dimensional representation passes through an output
    MLP that produces a residual correction $\Delta\hat{\mathbf{y}} \in
    \mathbb{R}^{9 \times H}$,
   added to the base forecast across all
   9 quantiles to yield
    $\hat{\mathbf{y}}_{\text{final}}$.  %
  }
  \label{fig:adapter}
\end{figure}


\subsection{Integration into Fully-Trained Transformers}
\label{sec:integration_trained}

Our fully-trained backbone is a standard Transformer encoder ($d_{\text{model}} =
256$, 6 layers, 8 heads, pre-norm), where each time step is embedded from $\mathbb{R}$ to $\mathbb{R}^{256}$ via a learned linear projection. Sinusoidal positional encodings are then added 
  to each token before the encoder. Both topology-derived features are injected as a global context
vector broadcast-added to every temporal token before the encoder
(Figure~\ref{fig:broadcast_injection}).

\paragraph{Global context injection.}
A context projection $\mathbf{W}_{\text{ctx}}$ maps the
   125-dim TDA fingerprint to $\mathbb{R}^{d_{\text{model}}}$. On datasets with an   
  explicit entity hierarchy (e.g., M5 store$\times$category), learned entity        
  embeddings are concatenated with the fingerprint before projection. The 256-dim   
  spectral coordinate is then mapped into the same space through a        
  dedicated projection $\mathbf{W}_{\text{sheaf}}$ and added to the result. 
  The two projections are kept separate intentionally. When $\mathbf{W}_{\text{sheaf}}$ is instead shared
   with $\mathbf{W}_{\text{ctx}}$ in a single joint linear layer, gradient descent
  tends to assign near-zero weights to the sheaf columns early in training,
  suppressing the sheaf signal before it can influence the model. A dedicated
   projection path prevents this. The resulting vector $\mathbf{g} \in \mathbb{R}^{d_{\text{model}}}$ is added to every
  temporal token $\mathbf{z}_t \in \mathbb{R}^{d_{\text{model}}}$ across all $L$ input steps:
 $$\mathbf{z}_t \leftarrow \mathbf{z}_t + \mathbf{g}, \quad t = 1,\ldots,L.$$

  Training minimizes a Huber quantile loss ($\delta = 1.0$, a standard choice robust to outliers) over 9 output
  quantiles. Calibration results appear in
  Appendix~\ref{app:calibration}. Full architecture and hyperparameter details appear in Appendix~\ref{app:arch}.

\subsection{Integration into Pre-Trained Foundation Models}
\label{sec:integration_tsfm}

For pre-trained backbones, we freeze all weights and train a lightweight
  \emph{topology adapter} that corrects the frozen base forecast. Since no gradient flows through the backbone, the adapter applies to any model that
   produces a point forecast.

\paragraph{Adapter architecture.}
The adapter processes four inputs through dedicated branches (Figure~\ref{fig:adapter}). Each branch
  projects to a common dimension of $128$, preventing any single input
  from dominating by sheer size. The four branches~are:
\begin{itemize}
  \item \textbf{TDA branch:} 125-dim population fingerprint, two-layer MLP
  with LayerNorm.
  \item \textbf{Sheaf branch:} 256-dim spectral coordinate, two-layer MLP with LayerNorm.
  \item \textbf{Context branch:} four z-scored series statistics (mean, standard
    deviation, linear trend slope, and last observed value), linear layer with
  LayerNorm.
  \item \textbf{Forecast branch:} cached median forecast $\hat{\mathbf{y}}_{\text{base}} \in \mathbb{R}^{H}$ from the
   frozen backbone, projected via linear layer with LayerNorm.
\end{itemize}
Z-scoring the context statistics removes cross-series scale variation, so the adapter
  learns meaningful patterns rather than unit conversions.

The adapter predicts a residual correction rather than a forecast from scratch, ensuring the model learns only the topological contribution. The four branch representations are concatenated and   
  passed through an output MLP $(512 \to 256 \to 9H)$:
\[
  \hat{\mathbf{y}}_{\text{final}}
  \;=\; \hat{\mathbf{y}}_{\text{base}}
  \;+\;
  \mathrm{OutputMLP}\!\bigl(
    [\mathbf{h}_{\text{tda}},\,\mathbf{h}_{\text{sheaf}},\,
     \mathbf{h}_{\text{ctx}},\,\mathbf{h}_{\text{fc}}]
  \bigr),
\]
 $\hat{\mathbf{y}}_{\text{base}}$ is broadcast across all quantiles as a warm start, with no gradient flowing through the backbone.

\paragraph{Ablations.}
Across the fully-trained and pre-trained settings, three architecture-matched configurations
  are evaluated:
\textbf{Vanilla} (no topology), \textbf{$+$TDA} 
(the population fingerprint), and \textbf{$+$TDA $+$ Sheaf} (population fingerprint and per-series spectral coordinates). \textbf{$+$TDA $+$ Sheaf} is the full TopoPrimer model.
Across all three variants, the output MLP is identical. Between variants, parameter differences reflect only the topology encoding branches, isolating topology's 
  contribution from additional prediction capacity.

\section{Results}
\label{sec:results}

\subsection{Topology Screening}
\label{sec:screening}

From the precomputed TDA features, we derive a simple pre-training screen:
$H_1/N$, the number of persistent loops in the domain divided by the number
of series. More loops per series means the correlation manifold has richer
cyclic co-movement structure, and therefore predicts a larger TDA contribution. The sheaf
coordinate is independent: it provides consistent per-series gains on every
domain regardless of loop density, and the screening criterion governs only
how much TDA will \emph{amplify} those gains.

Table~\ref{tab:screening} shows how $H_1/N$ predicts the magnitude of error reduction.
\textbf{METR-LA} and \textbf{ECL} share similar $H_1/N$ (0.22 and 0.26) and
similar modest gains ($-0.005$ and $-0.012$ MAE).
\textbf{Monash Weather} ($H_1/N{=}0.61$) stands out: its denser genuine loop
structure produces gains $5{\text{--}}14{\times}$ larger in MAE and
$20{\text{--}}48{\times}$ larger in MSE than ECL.
\textbf{M5 Household} has $H_1/N{=}4.12$, but the count is artifact-inflated:
shared weekly and annual seasonality creates calendar harmonics, not cross-series
relational loops, so TDA contributes near-zero and the observed MAE gain comes
from the sheaf alone.

\begin{table}[h]
\centering
\caption{%
  \textbf{\boldmath$H_1/N$ density characterizes domain manifold structure and
  predicts TDA amplification.}
  $H_1$ generators are persistent loops in the correlation manifold; $H_1/N$ normalizes
  by series count. TDA$+$Sheaf $\Delta$MAE (Chronos) scales with $H_1/N$: modest on
  sparse and artifact-inflated domains, strong on genuine-rich domains. Sheaf gains
  are present on all domains regardless of $H_1/N$.
  ($H_0$, $H_1$, $H_2$ landscape curves in Appendix~\ref{app:homology}).%
}
\label{tab:screening}
\begin{tabular}{llrrrr}
\toprule
\textbf{Dataset} & \textbf{Domain} & \textbf{N} & \textbf{\boldmath$H_1$}
  & \textbf{\boldmath$H_1/N$} & \textbf{TDA$+$Sheaf $\Delta$MAE} \\
\midrule
METR-LA        & Traffic     &      207 & 46         & 0.22 & $-0.005$ \\
Monash Weather & Weather     &  3{,}010 & 1{,}847    & 0.61 & $-0.074$ \\
ECL            & Electricity &      321 & 83         & 0.26 & $-0.012$ \\
M5 Household   & Retail      &  9{,}890 & 40{,}780\rlap{$^\dagger$} & 4.12\rlap{$^\dagger$} & $-0.015$ \\
\bottomrule
\multicolumn{6}{l}{\footnotesize $^\dagger$M5 $H_1$ inflated by shared weekly/annual calendar periodicity; genuine cross-series loop count unknown.}\\
\end{tabular}
\end{table}

UMAP projections (Figure~\ref{fig:umap_manifolds}) confirm why loop density varies across domains:
ECL and Weather display arc and loop structure; METR-LA shows a filament;
M5 shows a structureless diffuse cloud consistent with calendar-driven correlations
and no exploitable manifold geometry.

\subsection{Main Results}
\label{sec:main_results}

Table~\ref{tab:toposheaf_public} reports MAE and a domain-standard secondary
metric across all four benchmarks and three backbone families. Secondary metrics
follow the literature convention and were fixed before any topology model was
trained. In the discussion below, ``$+$TDA $+$ Sheaf'' refers to the full TopoPrimer model.  A consistent pattern emerges: the sheaf is the primary driver of gains across all domains, and TDA alone never improves over vanilla.
TDA alone lacks the per-series resolution to differentiate individual series. TDA is a population-level signal, and without sheaf coordinates to anchor it locally, it cannot know where in the population a given series sits.  We discuss each dataset in turn.

\begin{table}[t]
\centering
\caption{MAE and secondary metric across four public benchmarks. \textbf{Bold} = best per section. $\downarrow$ lower is better. In-table naming: ``$+$TDA $+$ Sheaf'' is the full TopoPrimer model. Secondary
metrics: METR-LA ($H{=}15$ steps at 5-min intervals, 207 sensors) uses MAPE per traffic forecasting convention.
ECL ($H{=}96$, 321 clients) and Monash Weather ($H{=}30$,
3{,}010 variates) use MSE to weight peak-error sensitivity. M5 ($H{=}4$, Household, 9{,}890 items) uses WAPE for scale-free cross-item comparison.}
\label{tab:toposheaf_public}
\resizebox{\linewidth}{!}{%
\small\setlength{\tabcolsep}{5pt}%
\begin{tabular}{lcccccccc}
\toprule
& \multicolumn{2}{c}{\textbf{METR-LA}}
& \multicolumn{2}{c}{\textbf{ECL}}
& \multicolumn{2}{c}{\textbf{Monash Weather}}
& \multicolumn{2}{c}{\textbf{M5}} \\
\cmidrule(lr){2-3}\cmidrule(lr){4-5}\cmidrule(lr){6-7}\cmidrule(lr){8-9}
\textbf{Model}
  & \textbf{MAE$\downarrow$} & \textbf{MAPE\%$\downarrow$}
  & \textbf{MAE$\downarrow$} & \textbf{MSE$\downarrow$}
  & \textbf{MAE$\downarrow$} & \textbf{MSE$\downarrow$}
  & \textbf{MAE$\downarrow$} & \textbf{WAPE$\downarrow$} \\
\midrule
\multicolumn{9}{l}{\textit{Transformer variants}} \\
Transformer                   & 2.206 & 3.812 & \textbf{0.193} & \textbf{0.091} & 2.175 & 25.935 & 1.866 & 0.264 \\
Transformer $+$ TDA           & 2.206 & 3.812 & 0.197 & 0.102 & 2.170 & 26.182 & 1.865 & 0.264 \\
Transformer $+$ TDA $+$ Sheaf & \textbf{2.203} & \textbf{3.809} & 0.196 & \textbf{0.091} & \textbf{2.004} & \textbf{25.143} & \textbf{1.827} & \textbf{0.259} \\
\midrule
\multicolumn{9}{l}{\textit{Chronos 2.0 variants}} \\
Chronos Zero-Shot             & 3.348 & 5.615 & 0.586 & 0.610 & 2.344 & 29.776 & \textbf{0.918} & \textbf{1.450} \\
Chronos Vanilla Adapter       & 2.383 & 4.087 & 0.302 & 0.205 & 2.015 & 28.487 & 1.040 & 1.643 \\
Chronos $+$ TDA               & 2.392 & 4.091 & 0.302 & 0.205 & 2.031 & 28.381 & 1.039 & 1.641 \\
Chronos $+$ TDA $+$ Sheaf     & \textbf{2.378} & \textbf{4.063} & \textbf{0.290} & \textbf{0.190} & \textbf{1.941} & \textbf{27.773} & 1.025 & 1.618 \\
\midrule
\multicolumn{9}{l}{\textit{TimesFM 2.5 variants}} \\
TimesFM Zero-Shot             & 2.441 & 4.200 & 0.580 & 0.602 & 2.032 & 28.032 & \textbf{0.914} & \textbf{1.443} \\
TimesFM Vanilla Adapter       & 2.355 & 4.058 & 0.300 & 0.204 & 2.038 & 28.173 & 1.037 & 1.636 \\
TimesFM $+$ TDA               & 2.356 & 4.064 & 0.300 & 0.204 & 2.067 & 28.212 & 1.034 & 1.632 \\
TimesFM $+$ TDA $+$ Sheaf     & \textbf{2.336} & \textbf{4.033} & \textbf{0.289} & \textbf{0.190} & \textbf{1.974} & \textbf{27.875} & 1.025 & 1.618 \\
\bottomrule
\end{tabular}%
}
\end{table}

\paragraph{METR-LA.}
TDA alone provides no lift and slightly degrades Chronos (MAE 2.383 $\to$ 2.392). This is
consistent with the sparse $H_1$ pre-screen verdict (Table~\ref{tab:screening}): injecting a near-empty topology fingerprint adds noise without useful structure.
The sheaf nonetheless retains a small consistent benefit even on this sparse
manifold. The full TopoPrimer model improves over Vanilla for every backbone, with
TimesFM reaching the best adapter result (MAE 2.355 $\to$ 2.336) and the Transformer the best
absolute result (MAE 2.203).

\paragraph{ECL.}
Topology gains on ECL are driven entirely by the sheaf. The vanilla Transformer achieves the lowest MAE overall (0.193). $+$TDA $+$ Sheaf improves MSE but slightly degrades MAE (0.193 $\to$ 0.196), consistent with a fully-trained model that has already internalized the domain's relational structure on this compact 321-series dataset. The frozen foundation model backbones lack this domain-specific exposure, so the
   sheaf provides a useful complement: $+$TDA $+$ Sheaf delivers consistent gains
  for both Chronos (MAE: 0.302 $\to$ 0.290) and TimesFM (MAE: 0.300 $\to$ 0.289).

\paragraph{Monash Weather.}
Chronos was pre-trained on the Monash corpus, placing this benchmark
in-distribution. Even in-distribution, Chronos $+$ TDA $+$ Sheaf achieves the best MAE across all
models (1.941), suggesting that in-distribution
  pre-training and topology are complementary. For both adapter families, TDA alone degrades relative to Vanilla
(Chronos MAE: 2.015 $\to$ 2.031; TimesFM MAE: 2.038 $\to$ 2.067), introducing conflicting signal without the per-series positional 
  grounding the sheaf provides. Adding the sheaf drives full recovery and further   
  improvement.
Under MSE (the primary Monash Weather metric), Transformer $+$ TDA $+$ Sheaf is the
best overall model (25.143).

\paragraph{M5.}
On M5, vanilla adapter training degrades from zero-shot performance for both TSFMs
(Chronos MAE: 0.918 $\to$ 1.040; TimesFM MAE: 0.914 $\to$ 1.037), consistent with
adapter overfitting on a calendar-dominated domain where the frozen backbone already
captures the main periodic structure.
$+$TDA alone changes MAE by at most $0.003$ across all backbones, confirming that
the artifact-inflated $H_1/N{=}4.12$ encodes no useful population-level signal. This is 
precisely what the screening criterion predicts (Table~\ref{tab:screening}).
$+$TDA $+$ Sheaf recovers consistent gains from the degraded adapter baselines, with
both TSFM backbones converging to MAE~1.025.
The Transformer, unaffected by adapter overfitting, shows a direct 2.1\% MAE
improvement ($1.866 \to 1.827$).

\paragraph{Cross-backbone synthesis.}
Across all benchmarks, sheaf coordinates are the primary driver of improvement;
  TDA alone provides no consistent improvement over vanilla and occasionally degrades.
  For the Transformer, gains scale with manifold richness, with 7.9\% MAE reduction
  on $H_1$-rich Monash Weather and only marginal gains on $H_1$-sparse METR-LA.
  For foundation model backbones, $+$TDA $+$ Sheaf consistently matches or beats
  the vanilla adapter on every domain. Chronos
  extracts the largest gain on ECL ($-$4.0\% MAE, $-$7.3\% MSE), with TimesFM
  close behind. Full per-horizon ECL results appear in Appendix~\ref{app:ecl_full}.

\subsection{Three Hard Regimes: Fine-Tuning Robustness, Seasonal Spikes, and Cold Start}
\label{sec:three_regimes}
Open benchmarks cannot answer three questions that matter in practice: does topology help when the backbone is already fine-tuned? Does it hold up under peak seasonal demand? And does it work at entity launch, when a series has no history?

 We evaluate on an internal dataset of
  $N{=}307{,}818$ active series across $4{,}575$  entities and
  $603$ items. The domain is large enough to test all three regimes and has a sparse manifold ($H_1{=}617$, $H_1/N{=}0.002$), placing it in the 
genuine-sparsity regime of
  Table~\ref{tab:screening}. Despite the sparse manifold, both sheaf coordinates and TDA contribute meaningful gains, confirming that each component is effective independently of manifold richness.

  We compute two TDA fingerprints from the internal corpus: an entity-manifold fingerprint (TDA$_E$) and an item-manifold fingerprint (TDA$_I$). On public benchmarks, +TDA corresponds to TDA$_E$
  alone, as item-resolution depth is insufficient to compute TDA$_I$. All three internal evaluations use both. 
  Fine-tuning robustness is evaluated using Chronos adapter families atop zero-shot and fine-tuned backbones.
  Seasonal-spike and cold-start evaluations use the
  fully-trained Transformer family. Across both families, architecture and training are held fixed so that topology is the sole variable.

\subsection{Fine-Tuning Robustness}
A natural concern is that fine-tuning on in-domain data should subsume any topological signal, making TopoPrimer redundant once the backbone is adapted. This hypothesis does not hold. 

We evaluate topology adapters atop both a frozen zero-shot Chronos checkpoint and a checkpoint fine-tuned on the internal corpus. Despite the two backbones differing substantially in domain adaptation, the topology gain is nearly identical: ($\Delta\mathrm{MAE} = {-}0.022$) on zero-shot Chronos and ($\Delta\mathrm{MAE} = {-}0.024$) on fine-tuned Chronos. Fine-tuning moves the baseline, but it does not absorb the topological signal.

The invariance is expected: the univariate fine-tuning objective has no mechanism to recover cross-series structural information. Topology and fine-tuning address different aspects of the problem and their benefits are additive. Full results appear in Appendix~\ref{app:ft_robustness}.

\subsection{Seasonal Spikes}
  \label{sec:seasonal_spikes}
The sharpest forecasting test in practice is peak seasonal demand, where distribution shift is both large and predictable in timing but not magnitude. We evaluate over the dataset's sharpest four-week
  annual demand window.
  Classical baselines and zero-shot TSFMs degrade substantially:
  XGBoost~\citep{chen2016xgboost} MAE rises from 2.272 to 3.368 (+48\%), DLinear~\citep{zeng2023transformers} from
  2.089 to 3.060 (+46\%), and Chronos zero-shot from 1.853 to
  2.780 (+50\%). The vanilla Transformer is the strongest
  non-topology baseline (+12\%), as training on full annual
  cycles gives it implicit seasonal knowledge. Even so, it
  enters the window already above every topology variant.

  Where all other models surge, topology variants remain nearly
  flat (Figure~\ref{fig:coldstart_mae_2}(b),
  Table~\ref{tab:seasonality_mae}).
  The best-performing model is Transformer$+$TDA$_E$$+$TDA$_I$$+$Sheaf, finishing with MAE 1.924: 43\% below XGBoost,
  31\% below Chronos, and 15\% below vanilla Transformer.
  The topological prior encodes the global co-movement structure
  of the item manifold explicitly, giving each model a stable
  geometric anchor as demand patterns~shift.

\begin{figure*}[htbp]
  \centering
  \includegraphics[width=0.95\linewidth]{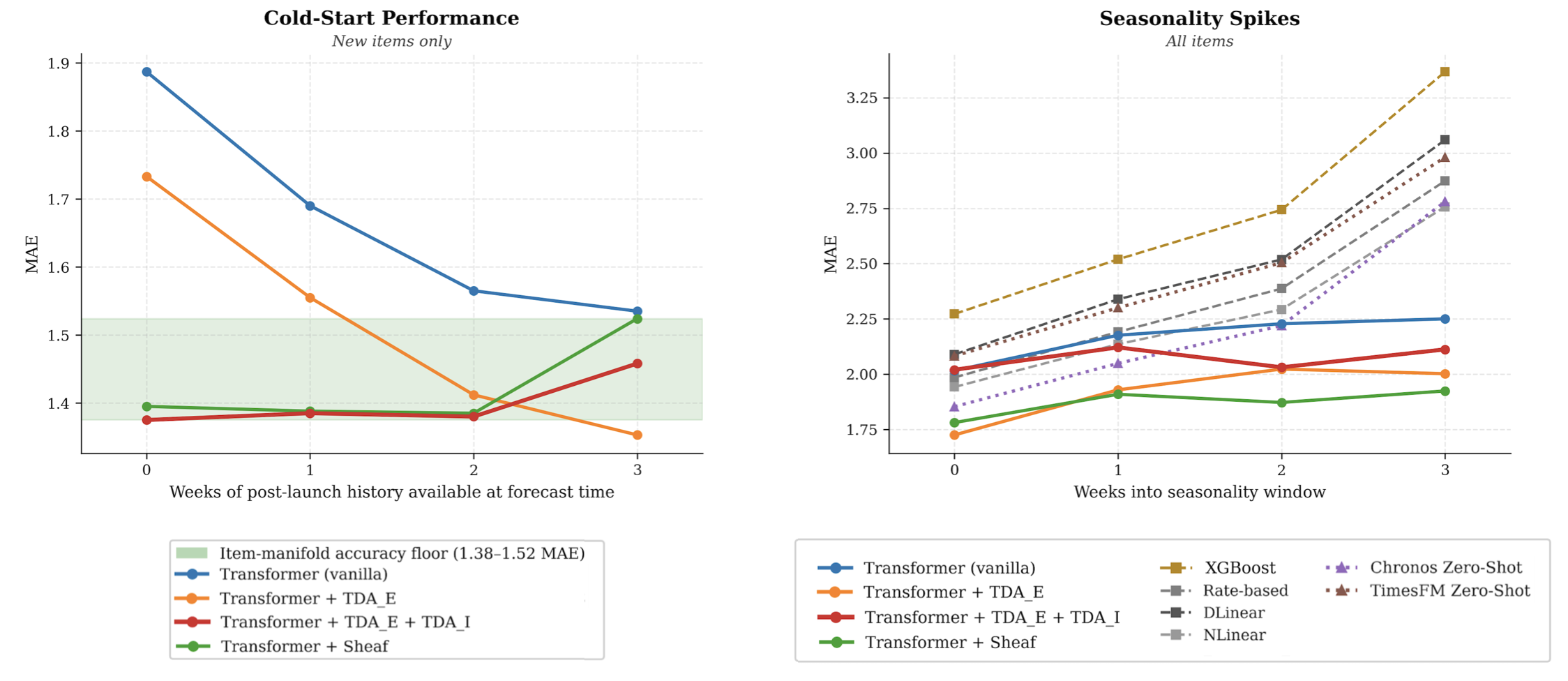}
   \caption{%
   \textbf{(a) Cold-start performance} (new items only, $N=40{,}324$ series):
    MAE vs.\ weeks of post-launch history available. All models receive a
    zero context at week~0; the shaded band (MAE 1.38--1.52) marks the
    accuracy floor TDA$_I$ topology variants maintain from launch, while vanilla
    Transformer (MAE~1.887 at week~0) never enters it despite training on
    the full population.
    \textbf{(b) Seasonality window} (all items): MAE over four weeks of
    peak demand. Classical baselines and zero-shot TSFMs degrade sharply
  (+46--50\%). Topology variants degrade only marginally, with
  Transformer$+$TDA$_E$$+$TDA$_I$$+$Sheaf maintaining the lowest MAE
  from week 1 onward, attributable to explicit cross-item seasonal
  structure encoded in the population manifold.%
  }
  \label{fig:coldstart_mae_2}
\end{figure*}

\subsection{Cold Start}
\label{sec:coldstart}
A new item has no history, but it has a position in the manifold, and with TopoPrimer that position is available from the first forecast step. At launch, every model receives a 52-week all-zero context window. Classical baselines and zero-shot TSFMs are effectively forecasting from zeros, making them null comparisons at week 0. The meaningful comparison is within the Transformer family, where architecture, training procedure, and temporal encoding are identical. Topology is the sole variable (Figure~\ref{fig:coldstart_mae_2}(a), Table~\ref{tab:cold_start_mae}).

The advantage is immediate. At week 0, before a single week of post-launch history exists,   Transformer$+$TDA$_E$$+$TDA$_I$ achieves MAE 1.375, and the Transformer$+$TDA$_E$$+$TDA$_I$$+$Sheaf variant achieves MAE 1.395. Both are 26--27\% below the vanilla Transformer (1.887). The vanilla Transformer has no way to locate a new item within the manifold, but topology supplies exactly this missing position from the first forecast step. As post-launch history accumulates, the vanilla Transformer improves steadily (1.887 $\to$ 1.535 MAE by week 3), converging toward the topology variants. This suggests that topology is filling the gap that history would otherwise fill. The sheaf variant remains the most consistent across all weeks, holding MAE below 1.40 from week 0 through week 2, while Transformer$+$TDA$_E$ achieves the best single result at week 3 (MAE 1.353).

\paragraph{Three regimes, one structural limit.}
The pattern is consistent across all three regimes: fine-tuning robustness, seasonal spikes, and cold start. Where training signal is sufficient, topology amplifies it. Where it is absent or misleading, topology substitutes for it. In every case, a topological prior precomputed from the full population unlocks a source of
   signal that training alone does not recover.

\section{Conclusion}
\label{sec:conclusion}

The topology of a series population, both its global shape and the relational
  position of each series within it, is a recoverable, frozen signal that no
  backbone-agnostic framework encodes as a topological prior.
TopoPrimer shows
that injecting this signal as a precomputed context vector is sufficient to close
gaps that per-series training signal alone does not resolve: cold-start,
  peak-demand windows, and robustness under fine-tuning.
Critically, the topology gain is backbone-agnostic and survives fine-tuning with
near-identical magnitude, which provides evidence that gradient descent on
per-series losses and population-level geometry are learning complementary
things.

Topology screening provides a lightweight pre-deployment diagnostic: the $H_1/N$
  density criterion identifies domains where the correlation manifold carries genuine
  cyclic structure and TDA will amplify the sheaf-driven gains, versus domains where
  improvement will be sheaf-driven alone.  The topological features themselves admit two natural extensions: (i) 
  a learned filtration metric that recovers topological structure obscured by Pearson
   distance; and (ii) multi-parameter persistent homology filtrating along scale and
  demand volatility jointly, to capture geometry that single-parameter summaries
  miss.


\paragraph{Broader impacts.}
Improved forecasting accuracy in energy, retail, and logistics reduces
over-production, inventory waste, and resource consumption. Because
TopoPrimer's topology features are precomputed once per domain and reused
across all inference, the marginal cost of the topological context is
negligible.

\paragraph{Limitations.}
The fine-tuning robustness, seasonal-spike, and cold-start evaluations in
Sections~4.3--4.6 rely on an internal corpus that cannot be released.
The topological features and adapter architecture are fully reproducible
on public data; we provide the complete evaluation protocol, hyperparameter
settings, and statistical tests (Appendices~\ref{app:ablations}
and~\ref{app:internal_coldstart}).

\section*{Acknowledgments}
We are deeply grateful to Mohammadreza Armandpour for his honest, thoughtful feedback and review sessions that shaped this work. We also thank Jordan Mittleman for his generous time and care in helping bring this paper to life.


\appendix

\newpage

\section{Mathematical Preliminaries}
\label{app:prelim}

\subsection{Simplicial Complexes and Filtrations}

A \emph{simplicial complex} $\mathcal{K}$ is a collection of simplices
(vertices, edges, triangles, and higher-dimensional analogues) closed under
taking faces: if $\sigma \in \mathcal{K}$ is a simplex and
$\tau \subseteq \sigma$ is a face, then $\tau \in \mathcal{K}$.
A \emph{filtration} is a nested sequence of complexes
$\emptyset \subseteq \mathcal{K}_{\varepsilon_1} \subseteq
\mathcal{K}_{\varepsilon_2} \subseteq \cdots \subseteq \mathcal{K}$
parameterised by a growing threshold $\varepsilon \geq 0$.
The Vietoris-Rips construction below is a canonical way to build
such a filtration from pairwise distances.

\subsection{Vietoris-Rips Filtration}

Given a finite set of points $\mathcal{P}$ with pairwise distances $d$, the
\emph{Vietoris-Rips complex} at scale $\varepsilon$ admits every subset of
$\mathcal{P}$ whose pairwise distances all fall within $\varepsilon$:
\[
  \mathrm{VR}(\mathcal{P}, \varepsilon)
  = \bigl\{\, \sigma \subseteq \mathcal{P}
    \;\big|\; \max_{p,q \in \sigma} d(p,q) \leq \varepsilon \bigr\}.
\]
We set $\mathcal{P} = \{\mathbf{x}_1,\ldots,\mathbf{x}_N\}$ where each point is
one series and $d_{ij} = 1 - \lvert\rho_{ij}\rvert$ is the Pearson correlation distance.
This yields diagrams over the \emph{series population}, revealing which series cluster
($H_0$), which cyclic co-movements exist ($H_1$), and which structural boundaries
separate regimes ($H_2$). These are population-level descriptors, not per-series properties.

\subsection{Persistent Homology}

Persistent homology tracks how the topology of a Vietoris-Rips filtration
changes as $\varepsilon$ increases.
The $k$-th homology group $H_k(\mathcal{K})$ counts independent $k$-dimensional
holes. \emph{Persistent homology}~\citep{carlsson2009topology} tracks how homology
classes are born and die as $\varepsilon$ increases. Each feature is recorded as a
birth-death pair $(b,d)$, forming the persistence diagram $\mathrm{PD}_k$.
Features with large persistence $d-b$ represent robust structural properties;
short-lived features are noise.

\subsection{Persistence Landscapes}

The \emph{persistence landscape}~\citep{bubenik2015statistical} maps
$\mathrm{PD}_k$ to a fixed-size vector:
\[
  \lambda_k^{(\ell)}(t) = \ell\text{-th largest value of}\;
  \bigl\{(\min(t-b,\,d-t))_{+} \;:\; (b,d)\in\mathrm{PD}_k\bigr\},
\]
where $(\cdot)_{+} = \max(0,\,\cdot)$.
Evaluating on a fixed grid yields an $r$-dimensional vector that is
Lipschitz-stable with respect to the bottleneck distance $d_B$
(Appendix~\ref{app:lipschitz}).

\subsection{Cellular Sheaves}

A \emph{cellular sheaf} $\mathcal{F}$ over a graph $\mathcal{G} =
(\mathcal{V}, \mathcal{E})$ assigns a stalk
$\mathcal{F}_v \cong \mathbb{R}^d$ to each node and a linear
\emph{restriction map} $\mathcal{F}_{v \trianglelefteq e}$ to each node-edge
incidence. The \emph{sheaf Laplacian} $\mathbf{L}_\mathcal{F} = \delta^\top\delta$
penalizes deviation from global consistency:
\[
  \mathbf{x}^\top \mathbf{L}_\mathcal{F} \mathbf{x}
  = \sum_{(u,v)\in\mathcal{E}} w_{uv}
    \bigl\|\mathcal{F}_{u\trianglelefteq e}(\mathbf{x}_u)
    - \mathcal{F}_{v\trianglelefteq e}(\mathbf{x}_v)\bigr\|^2.
\]
For the spectral encoder used in all
  reported results, the restriction maps are
  identity and the harmonic section is spanned by
  the leading left singular vectors of the
  entity-time matrix, requiring no training. Section~\ref{sec:sheaf} describes both encoder
  variants and the integration in full.

\paragraph{Derivation: identity maps reduce to leading singular vectors.}
Setting all restriction maps to the identity,
$\mathcal{F}_{v \trianglelefteq e} = I$, the sheaf Laplacian
reduces to the standard weighted graph Laplacian
$\mathbf{L}_\mathcal{F} = \mathbf{D} - \mathbf{W}$,
where $\mathbf{W}_{uv} = w_{uv}$ is the edge-weight matrix
and $\mathbf{D}$ is the degree matrix.

We build the graph with weights $w_{uv} = |\rho_{uv}|$
(absolute Pearson correlations between entity time series).
The Pearson correlation $\rho_{uv}$ equals the cosine similarity of
$L^2$-normalised rows of the entity-time matrix
$\mathbf{M} \in \mathbb{R}^{N \times T}$, so the weight matrix
satisfies $\mathbf{W} \approx \tfrac{1}{T}|\mathbf{M}\mathbf{M}^\top|$
(element-wise absolute value) up to row-normalisation.

The \emph{harmonic section} minimises the consistency loss
over a $k$-dimensional subspace:
\[
  \min_{\mathbf{X} \in \mathbb{R}^{N \times k},\;
        \mathbf{X}^\top\mathbf{X}=I}
  \mathrm{tr}\!\left(\mathbf{X}^\top \mathbf{L}_\mathcal{F}
  \mathbf{X}\right).
\]
Substituting $\mathbf{L}_\mathcal{F} = \mathbf{D} - \mathbf{W}$
gives $\mathrm{tr}(\mathbf{X}^\top\mathbf{D}\mathbf{X}) -
\mathrm{tr}(\mathbf{X}^\top\mathbf{W}\mathbf{X})$.
For a $d$-regular graph, $\mathrm{tr}(\mathbf{X}^\top\mathbf{D}\mathbf{X})
= d\,\mathrm{tr}(\mathbf{X}^\top\mathbf{X}) = dk$
is constant under the orthonormality constraint, reducing the minimization to
\[
  \max_{\mathbf{X}^\top\mathbf{X}=I}
  \mathrm{tr}\!\left(\mathbf{X}^\top \mathbf{W} \mathbf{X}\right).
\]
By the Rayleigh--Ritz theorem, the solution is
$\mathbf{X} = \mathbf{U}_k$, the matrix of leading eigenvectors
of $\mathbf{W}$. Since $\mathbf{W} \approx
\tfrac{1}{T}|\mathbf{M}\mathbf{M}^\top|$ (element-wise), and in
domains with predominantly positive cross-series correlations
$|\mathbf{M}\mathbf{M}^\top| \approx \mathbf{M}\mathbf{M}^\top$,
the leading eigenvectors of $\mathbf{W}$ are well-approximated by the
leading left singular vectors of
$\mathbf{M}$ (via $\mathbf{M}\mathbf{M}^\top \mathbf{u}_i =
\sigma_i^2 \mathbf{u}_i$). The harmonic section of the
identity-restriction sheaf is therefore spanned by the
truncated SVD of the entity-time matrix, requiring
no gradient-based training.

\section{Lipschitz Stability of Persistence Landscape Features}
\label{app:lipschitz}

\begin{theorem}[Lipschitz Stability]
Let $c$ and $c'$ be any two entity populations. Let $\mathbf{h}_c, \mathbf{h}_{c'} \in \mathbb{R}^{125}$ be their respective persistence landscape fingerprints and $\mathcal{D}_c, \mathcal{D}_{c'}$ their persistence diagrams. Then:
\[
  \|\mathbf{h}_c - \mathbf{h}_{c'}\|_2 \;\leq\; C \cdot d_B(\mathcal{D}_c,\,\mathcal{D}_{c'})
\]
where $\|\cdot\|_2$ denotes the Euclidean norm, $d_B$ is the bottleneck distance between persistence diagrams, and $C$ depends only on the landscape sampling grid.
\end{theorem}

\textit{Proof sketch.}
The persistence landscape $\lambda_1$ is $1$-Lipschitz with respect to
the bottleneck distance~\citep{bubenik2015statistical}.
Stacking over $H_0, H_1, H_2$ introduces at most an $\ell^2$-to-supremum
factor of $\sqrt{125}$.~$\square$

\medskip

This stability result has three direct consequences for TopoPrimer.

\begin{enumerate}[noitemsep, topsep=2pt]
  \item \textbf{Noise robustness.}
    Missing weeks or measurement errors produce proportionally bounded
    perturbations to $\mathbf{h}_c$.

  \item \textbf{Cold-start coverage.}
    At launch, a new item has no individual history but receives three topology signals from TopoPrimer: a shared entity-manifold fingerprint (TDA$_E$), an item-manifold descriptor (TDA$_I$), and a per-item sheaf coordinate approximated from relational neighbors. All three are precomputed offline and require no per-item history. TDA$_E$ provides population-level structural context but does not place a new item within the item manifold; TDA$_I$ supplies that item-level positioning from day one (Section~\ref{sec:coldstart}).

  \item \textbf{Backbone agnostic.}
    The TDA fingerprint is precomputed once from the population and frozen---its $\ell_2$ norm is fixed at computation time and never updated during training or inference. A fixed-norm input introduces no gradient-driven drift, so it propagates as a reliable conditioning signal regardless of downstream model architecture. This provides a theoretical basis for consistent topology gains across backbone families.
\end{enumerate}

\clearpage
\section{Transformer Architecture}
\label{app:arch}

On the public benchmarks, three variants are evaluated: Vanilla, $+$TDA, and $+$TDA$+$Sheaf. All three share the same encoder and head and differ only in which topology block is active. On the internal corpus, a fourth variant stacks two topology blocks ($+$TDA$_E$$+$TDA$_I$$+$Sheaf); Section~\ref{sec:three_regimes} details this extended set.

\paragraph{Global context vector.}
For the public benchmarks, a context vector is formed by projecting the temporal features and the active topology block to $d_\text{model}{=}256$ (Table~\ref{tab:context-emb-public}). For the internal corpus, learned node, item, and category embeddings are concatenated with the temporal features before this projection (Table~\ref{tab:context-emb-internal}). The TDA fingerprint is the primary input to the main projection $\mathbf{W}_\text{ctx}$. The sheaf block, when active, receives a dedicated $\mathrm{Linear}(256{\to}256)$ projection whose output is summed with the main projection before broadcasting. A dedicated path is necessary because, when the sheaf coordinate shares a single joint projection with the TDA and entity embedding inputs, gradient descent assigns near-zero weights to the sheaf columns early in training, suppressing its contribution. The TDA fingerprint enters $\mathbf{W}_\text{ctx}$ as the primary input and does not face this suppression.

\begin{table}[H]
\centering\small
\caption{\textbf{Context vector components: public benchmarks.}
  Exactly one topology block is active per variant, or none for the vanilla baseline.}
\label{tab:context-emb-public}
\setlength{\tabcolsep}{6pt}
\begin{tabular}{lrl}
\toprule
\textbf{Feature} & \textbf{Dim} & \textbf{Notes} \\
\midrule
Temporal features  & 8 & Time-of-day and day-of-week encodings \\
\midrule
\textit{Topology (one active):} & & \\
\quad TDA block  & 125 & $H_0{+}H_1{+}H_2$ persistence landscape \\
\quad Sheaf block  & 256 & Per-series spectral relational feature \\
\bottomrule
\end{tabular}
\end{table}

\begin{table}[H]
\centering\small
\caption{\textbf{Context vector components: internal corpus.}
  Learned entity embeddings are concatenated with temporal features before the main projection.
  Topology blocks are stacked across variants; see Section~\ref{sec:three_regimes} for the full variant list.}
\label{tab:context-emb-internal}
\setlength{\tabcolsep}{6pt}
\begin{tabular}{lrl}
\toprule
\textbf{Feature} & \textbf{Dim} & \textbf{Notes} \\
\midrule
Node embedding   & 256 & Learned lookup, one per graph node \\
Item embedding   & 256 & Learned lookup, one per item type  \\
Category embedding & 64 & Learned lookup, all category levels \\
Temporal features  & 16 & Week-of-year sinusoid + fiscal-quarter one-hot \\
\midrule
\textit{Topology (stackable):} & & \\
\quad TDA$_E$ block  & 125 & Entity-manifold $H_0{+}H_1{+}H_2$ landscape \\
\quad TDA$_I$ block  & 125 & Item-manifold $H_0{+}H_1{+}H_2$ landscape \\
\quad Sheaf block  & 256 & Per-series spectral relational feature \\
\bottomrule
\end{tabular}
\end{table}

\paragraph{Encoder and head.}
Each demand scalar is embedded by $\mathrm{Linear}(1{\to}256)$, added to the
broadcast context, and augmented with sinusoidal positional encoding. The encoder
comprises six identical pre-norm layers, each with 8-head scaled dot-product
attention (head dim 32), a feed-forward network (hidden dimension 1024, GELU), and
dropout 0.10. The forecast head decodes the final-position representation:
$\mathrm{Linear}(256{\to}256) \to \mathrm{ReLU} \to \mathrm{Dropout}(0.1)
\to \mathrm{Linear}(256{\to}H{\times}Q)$, producing 9 quantile estimates
($\mathcal{Q}{=}\{0.02,0.10,0.20,0.30,0.50,0.70,0.80,0.90,0.98\}$) at each of $H$
forecast steps, where $H$ is the dataset's horizon (Table~\ref{tab:dataset-arch}).
Total parameters: ${\approx}8$M (internal corpus; smaller for public benchmarks
due to reduced entity embedding tables).

\paragraph{Optimizer and scheduler.}
All Transformer variants are trained with AdamW ($\beta_1{=}0.9$,
$\beta_2{=}0.999$, $\epsilon{=}10^{-8}$, weight decay $10^{-3}$, learning rate $10^{-4}$)
and a OneCycleLR scheduler (maximum learning rate $3{\times}10^{-4}$).
Topology adapters on frozen backbones use the same AdamW settings with
learning rate $3{\times}10^{-4}$ and a CosineAnnealingLR scheduler
(period $=$ number of epochs, minimum learning rate $= 0.1{\times}$ learning rate).
Loss is Huber quantile loss ($\delta{=}1.0$) over 9 output quantiles for all variants.

\paragraph{Dataset-specific settings.}
Table~\ref{tab:dataset-arch} lists the context window and forecast horizon for each
dataset. All variants share the encoder architecture above; only the sequence length,
horizon, and temporal feature dimensionality differ.

\begin{table}[H]
\centering\small
\caption{\textbf{Context and forecast horizon per dataset.}
  Context windows follow established benchmark protocols: ECL uses 96\,hr (4 days),
  matching the standard long-term forecasting evaluation~\citep{wu2021autoformer};
  METR-LA uses 12 steps (60\,min), the standard traffic forecasting protocol~\citep{li2018dcrnn}.
  ECL trains one checkpoint per horizon; all other datasets use a single horizon.
  METR-LA steps are 5-minute intervals.
  $^\dagger$Internal temporal features include week-of-year sinusoids and
  fiscal-quarter one-hot encoding; public datasets use time-of-day and
  day-of-week encodings (8\,dims).}
\label{tab:dataset-arch}
\setlength{\tabcolsep}{6pt}
\begin{tabular}{lrrl}
\toprule
\textbf{Dataset} & \textbf{Context} & \textbf{Horizon} & \textbf{Temp.\ dim} \\
\midrule
Internal corpus & 52 weeks  & 4 weeks                 & $16^\dagger$ \\
ECL             & 96 hr     & 96 / 192 / 336 / 720 hr & 8  \\
M5 (Walmart)    & 52 weeks  & 4 weeks                 & 8  \\
METR-LA         & 12 steps  & 15 steps                & 8  \\
Weather         & 300 days  & 30 days                 & 8  \\
\bottomrule
\end{tabular}
\end{table}

\section{Architecture Diagrams}
\label{app:arch_diagrams}

\begin{figure}[H]
  \centering
  \includegraphics[width=0.88\linewidth]{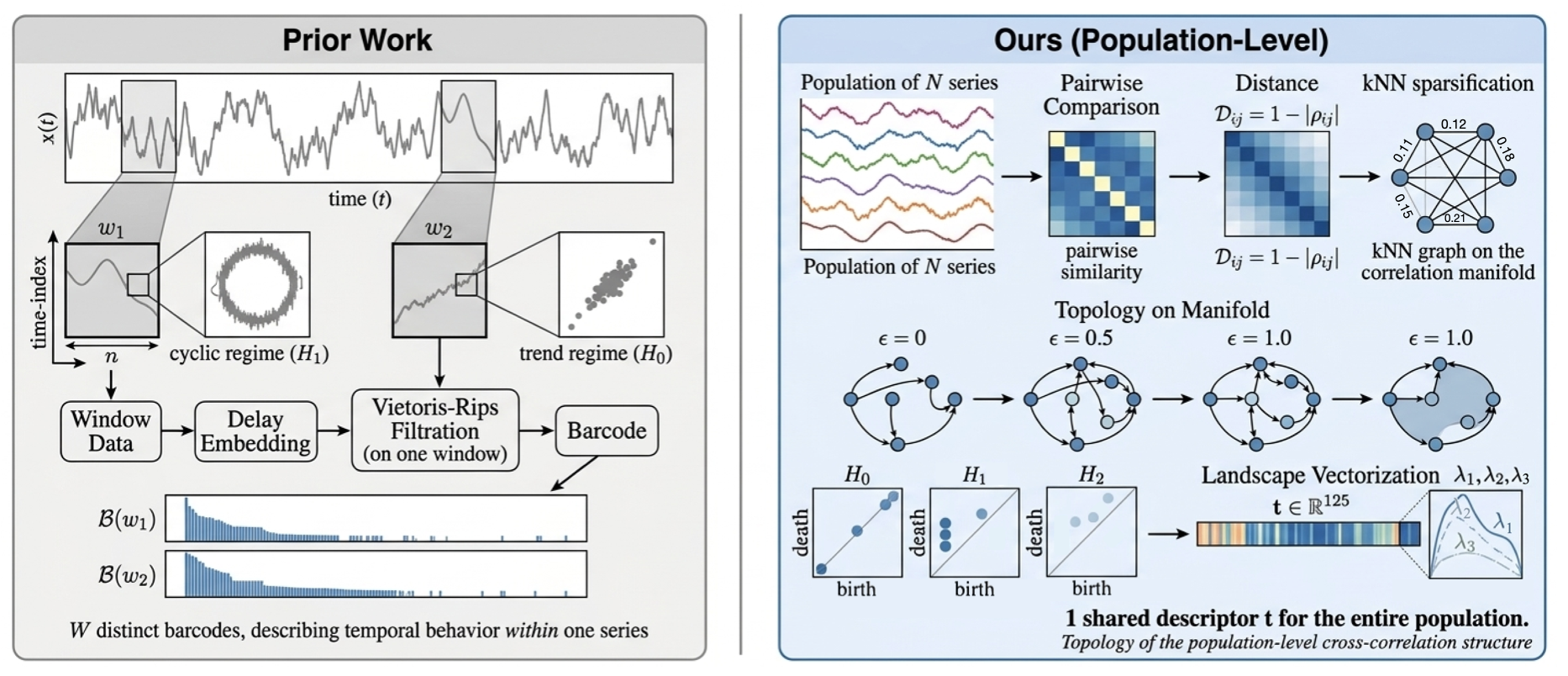}
  \caption{%
    \textbf{Per-series vs.\ population-level TDA.}
    Prior work (left) computes a separate vector per sliding window, yielding
    one descriptor per window of temporal dynamics \emph{within} one series. TopoPrimer (right) treats the full population as a point cloud, runs a single Vietoris-Rips filtration on the correlation manifold, and produces one
    125-dimensional persistence landscape vector shared across all series. TopoPrimer encodes structure
    that no individual trajectory contains.%
  }
  \label{fig:population_tda}
\end{figure}

Figure~\ref{fig:population_tda} illustrates the core methodological departure: prior work computes per-series vectors, while TopoPrimer applies a single filtration to the full population manifold. Figure~\ref{fig:sheaf_spectral} shows the complementary sheaf coordinate pipeline, which produces a 256-dimensional per-series spectral coordinate via truncated SVD of the entity-time matrix.

\begin{figure}[!ht]
  \centering
  \includegraphics[width=0.80\linewidth]{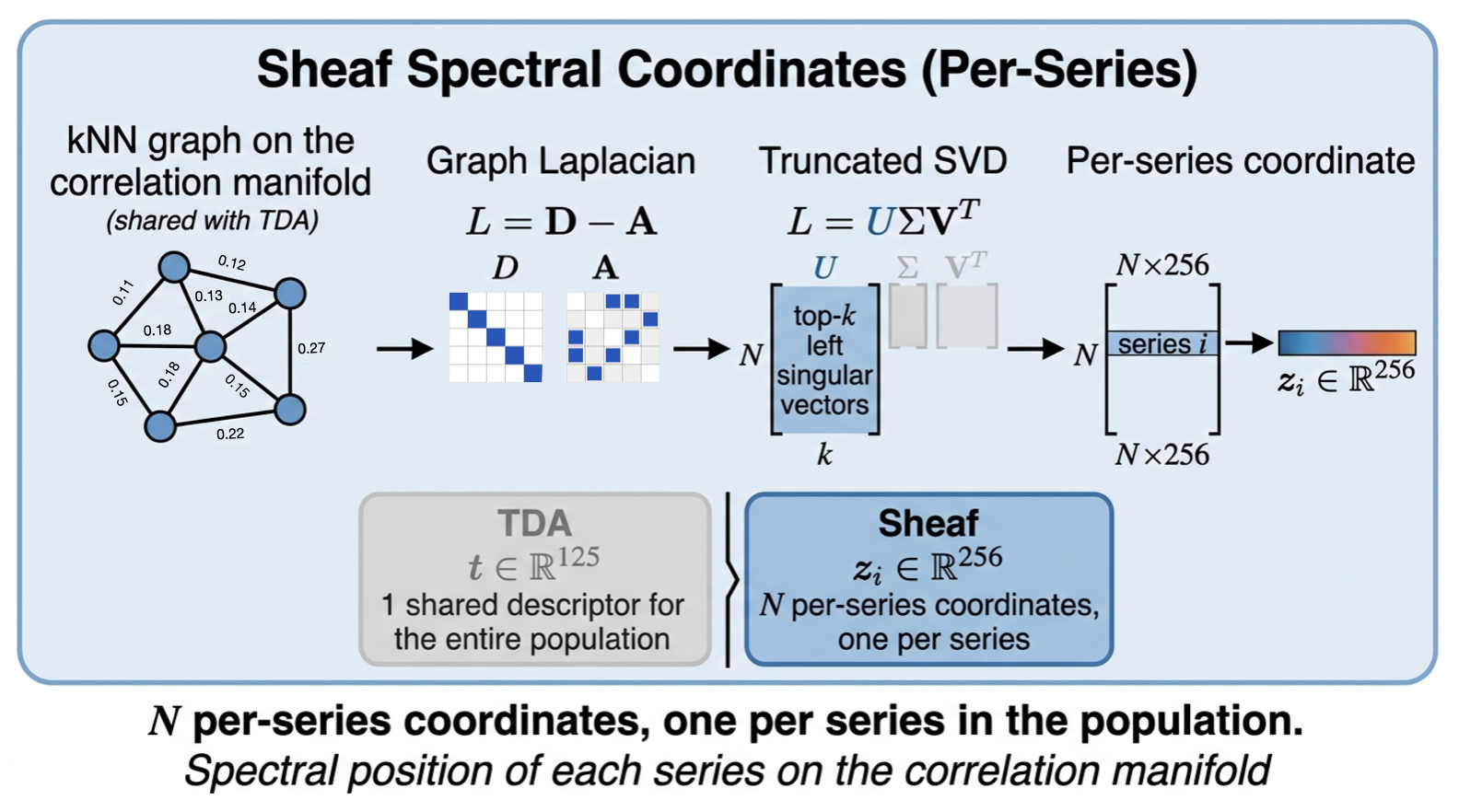}
  \caption{%
    \textbf{Sheaf spectral coordinate computation (per-series).}
    The kNN graph on the correlation manifold (shared with TDA) defines
    the edge weights used to form the entity-time matrix
    $\mathbf{X} \in \mathbb{R}^{N \times T}$.  A truncated SVD
    $\mathbf{X} \approx U\Sigma V^\top$ extracts the top-$k$ left singular vectors; row $i$ of
    $U$ is the 256-dimensional spectral coordinate $\mathbf{z}_i \in
    \mathbb{R}^{256}$ for series $i$. This yields $N$ per-series coordinates,
    one per series in the population, encoding each series' structural
    position on the correlation manifold. In contrast to the TDA fingerprint
    $\mathbf{t} \in \mathbb{R}^{125}$ (one shared vector for the entire
    population), the sheaf coordinate is per-series and relational.%
  }
  \label{fig:sheaf_spectral}
\end{figure}

Figure~\ref{fig:sheaf_spectral} illustrates how the same kNN graph that feeds
the TDA filtration yields a complementary per-series output: each series receives
a unique spectral coordinate $\mathbf{z}_i \in \mathbb{R}^{256}$ encoding its
structural position on the manifold, while the TDA fingerprint $\mathbf{t} \in
\mathbb{R}^{125}$ remains a single shared vector for the entire population.

\clearpage
\section{TDA Analysis of Public Benchmarks}
\label{app:homology}
\setlength{\floatsep}{4pt plus 1pt minus 1pt}
\setlength{\textfloatsep}{6pt plus 1pt minus 2pt}
\setlength{\abovecaptionskip}{4pt}
\setlength{\belowcaptionskip}{2pt}

The population-level TDA pipeline is illustrated in Figure~\ref{fig:population_tda} (Appendix~\ref{app:arch_diagrams}).
Most multi-panel figures in this section use a 2$\times$2 layout: ECL (top-left),
Monash Weather (top-right), M5 Household (bottom-left), METR-LA (bottom-right).
We place the two topology-rich benchmarks alongside the two null or sparse cases
for direct comparison.

Figure~\ref{fig:tda_landscape} overlays the persistence landscape vectors for all
four public benchmarks. $H_0$ (connected components) is broadly similar across
all datasets. Each population clusters into a small number of dominant groups at
coarse scales, contributing little discriminating signal. The diagnostic
information resides in $H_1$. Weather and ECL exhibit irregular, multi-scale peaks,
the signature of genuine cyclic co-movement distributed across the filtration
range. M5 instead shows evenly-spaced harmonic peaks consistent with calendar
repetition (7-day and 52-week periodicities shared by every item), not relational
geometry. These loops encode calendar artifact rather than manifold structure,
which is why the TDA fingerprint provides no useful grouping signal there.
METR-LA falls between these cases: ring roads and interchanges do produce a
non-trivial $H_1$ count (39--62 generators), but the network is predominantly
tree-like, so those cycles reflect road geometry rather than correlated demand
patterns. The $H_2$ panel (structural voids) shows meaningful activity only for
Weather and ECL, independently confirming that the fingerprint carries genuine
multi-dimensional manifold signal on those benchmarks. Taken together, these
landscape signatures are the visual basis for the pre-screening criterion in
Section~\ref{sec:screening}: the fingerprint distinguishes topology-rich from
topology-poor domains before any model is trained.

\begin{figure}[!ht]
  \centering
  \includegraphics[width=\textwidth]{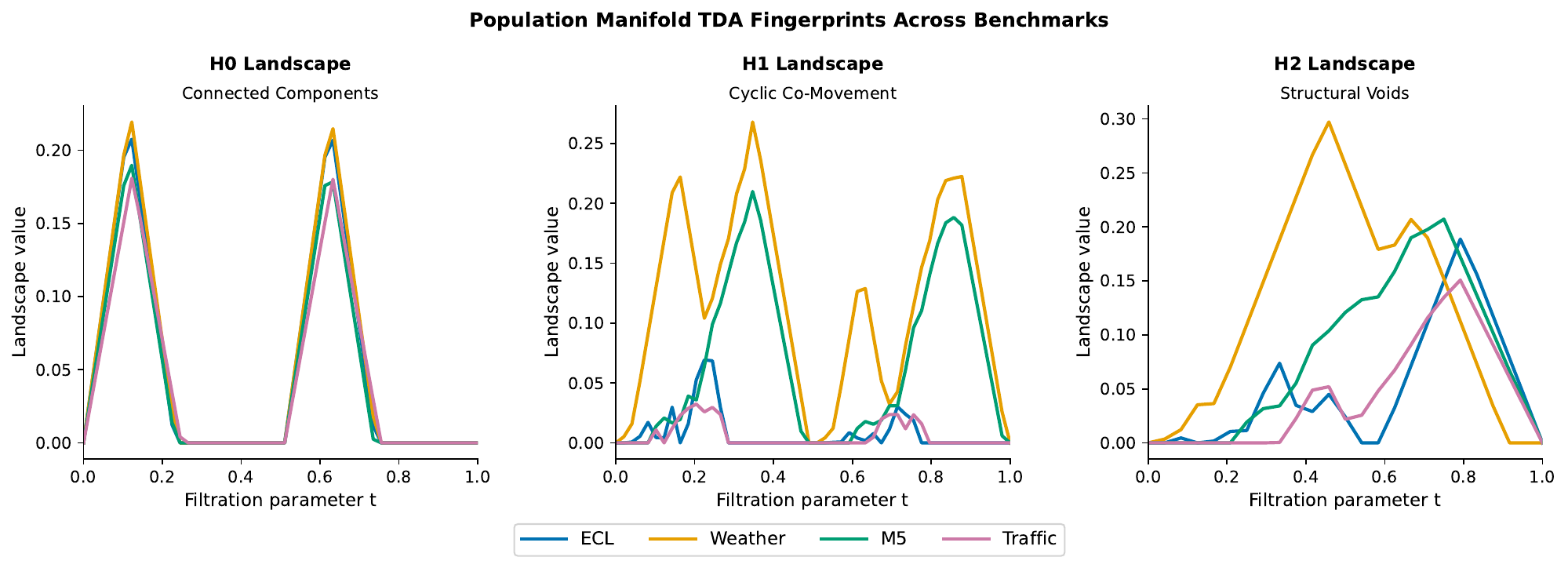}
  \caption{%
    \textbf{Population manifold TDA fingerprints across benchmarks.}
    Each curve is the 125-dimensional persistence landscape vector
    ($H_0$: 50 dims, $H_1$ : 50 dims, $H_2$: 25 dims) for one dataset.
    \textbf{$H_0$} (connected components): similar across all datasets; coarse
    cluster merging dominates the signal.
    \textbf{$H_1$ } (cyclic co-movement): the diagnostic panel. Weather and ECL
    exhibit irregular, multi-scale peaks indicative of genuine cyclic manifold
    structure. M5 shows evenly-spaced harmonics (calendar artifact, not relational
    geometry). METR-LA is near-flat (tree-like road hierarchy, few genuine cycles).
    \textbf{$H_2$} (structural voids): active only for Weather and ECL.%
  }
  \label{fig:tda_landscape}
\end{figure}

\subsection{Cross-Segment Comparison}
\label{app:cross_segment}

Figure~\ref{fig:cross_segment_all} shows how TDA feature vectors differ across
temporal or demographic splits within each dataset, alongside $H_1$ Wasserstein-2
distance matrices quantifying topological dissimilarity between segments. ECL and Weather exhibit
meaningful topology differences across their primary splits (weekday/weekend and
all-days/active-days respectively), confirming that the segments capture structurally
distinct regimes. M5 and METR-LA show weaker or noise-driven cross-segment variation,
consistent with calendar artifact and sparse topology respectively.

\begin{figure}[!ht]
  \centering
  \begin{subfigure}[t]{0.47\textwidth}
    \includegraphics[width=\linewidth]{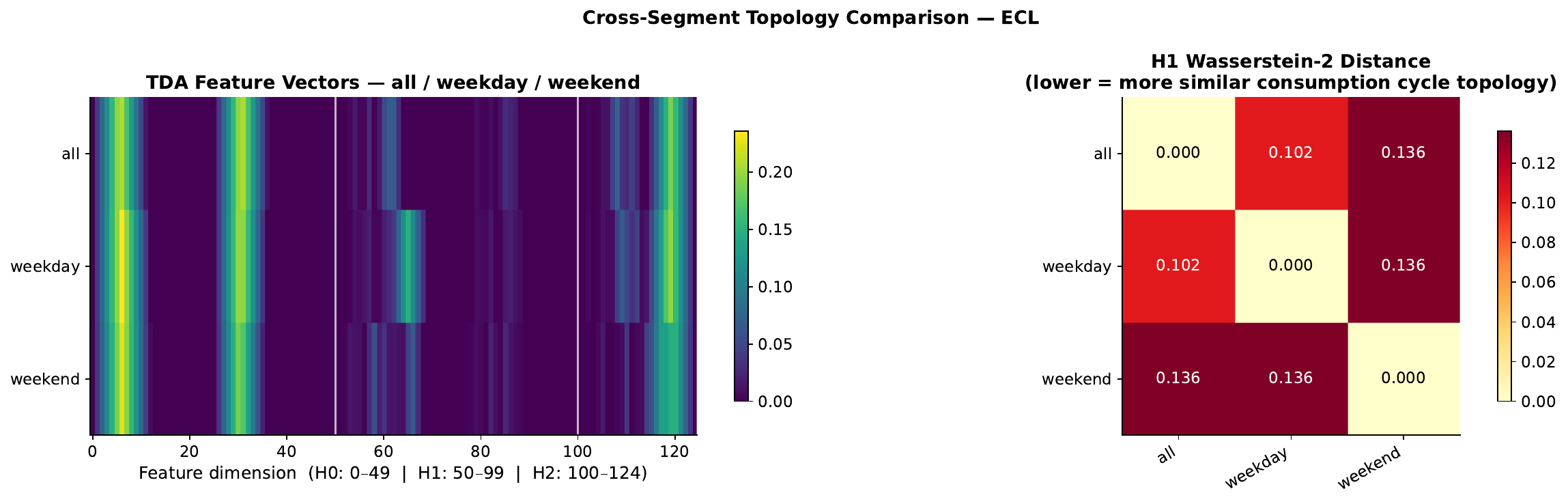}
    \caption{\textbf{ECL.} Weekday segments show higher $\beta_1$ and larger $H_1$ landscape norms than weekend, reflecting a genuine structural regime change driven by commercial usage cycles.}
    \label{fig:cs_ecl}
  \end{subfigure}
  \hfill
  \begin{subfigure}[t]{0.47\textwidth}
    \includegraphics[width=\linewidth]{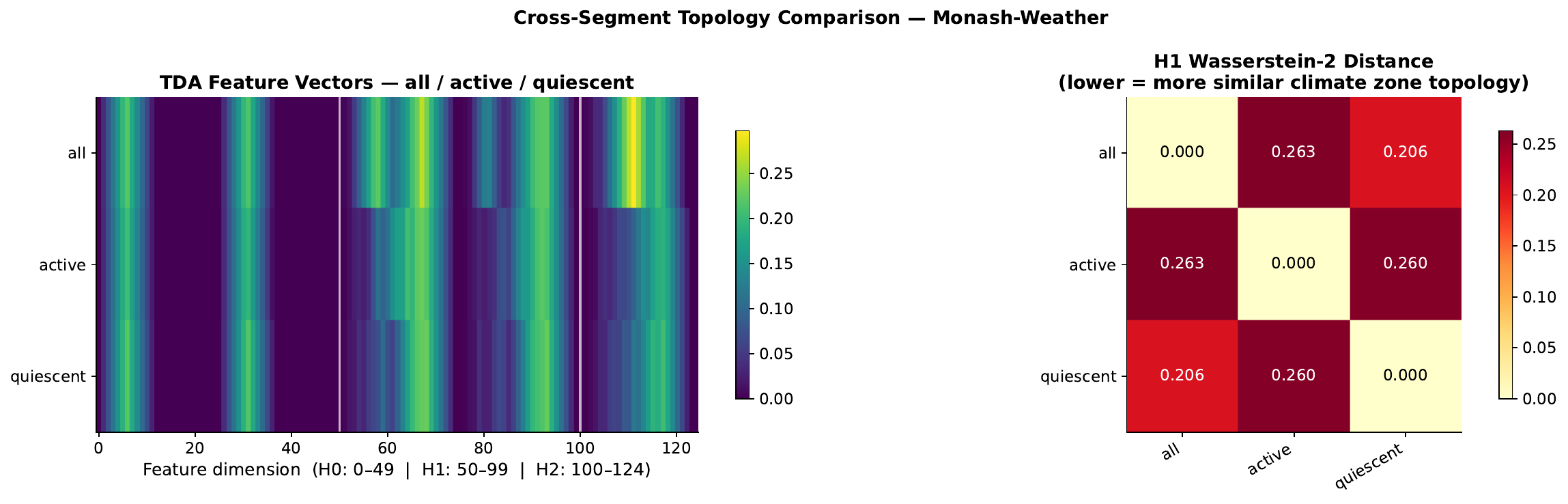}
    \caption{\textbf{Monash Weather.} Active-day segments show higher $\beta_1$ counts and larger $H_1$ landscape norms, confirming that missing observations weaken but do not eliminate the topological signal.}
    \label{fig:cs_weather}
  \end{subfigure}
  \\[3pt]
  \begin{subfigure}[t]{0.47\textwidth}
    \includegraphics[width=\linewidth]{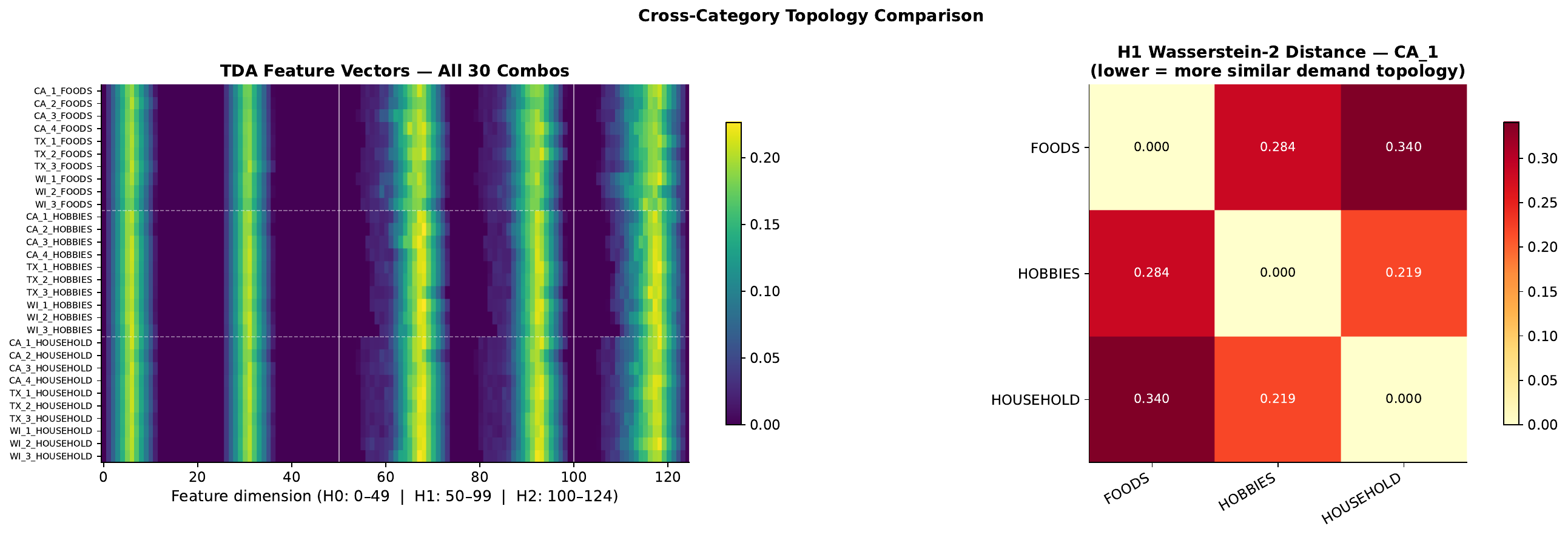}
    \caption{\textbf{M5 Household.} Cross-category TDA variation is modest relative to within-category calendar artifact, with small Wasserstein distances reflecting shared periodicity across all categories.}
    \label{fig:cs_m5}
  \end{subfigure}
  \hfill
  \begin{subfigure}[t]{0.47\textwidth}
    \includegraphics[width=\linewidth]{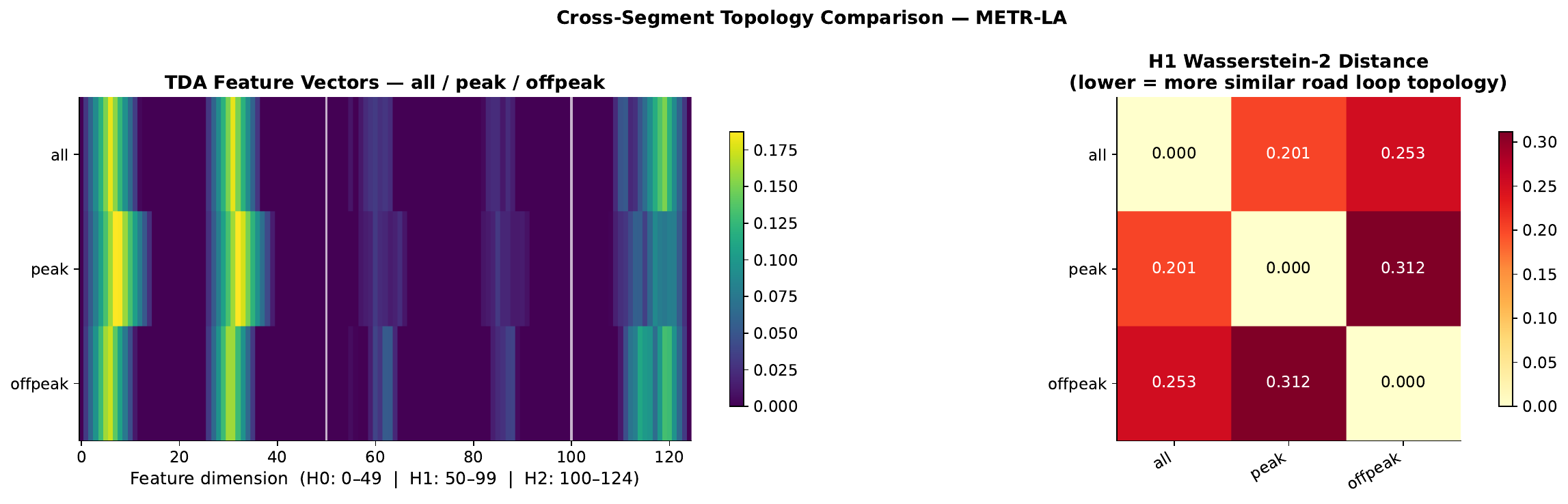}
    \caption{\textbf{METR-LA.} The off-peak segment shows higher $\beta_1$ counts and the largest Wasserstein-2 distance from peak, confirming that time-of-day fundamentally alters sensor correlation topology.}
    \label{fig:cs_traffic}
  \end{subfigure}
  \caption{\textbf{Cross-segment TDA comparison across all four public benchmarks.}
    Left panel of each subfigure: 125-dim TDA feature vectors per segment ($H_0$:
    dims~0--49, $H_1$: dims~50--99, $H_2$: dims~100--124). Right panel: $H_1$ Wasserstein-2
    pairwise distance matrix between segments (lower~=~more topologically similar).
    ECL and Weather show structurally meaningful cross-segment differences;
    M5 cross-category variation is dominated by shared calendar artifact;
    METR-LA shows the most pronounced peak/off-peak topological divergence.}
  \label{fig:cross_segment_all}
\end{figure}

\addvspace{10ex}

\subsection{Manifold Structure (UMAP Projection)}
\label{app:umap}

Figure~\ref{fig:umap_manifolds} shows UMAP 2D projections of each entity
correlation manifold, visualizing the global geometric structure that persistent
homology quantifies. The shape of the projection
encodes the topological verdict. An arc or loop shape is the 2D signature of
$\beta_1 > 0$, a diffuse cloud indicates low topological structure, and a filament
indicates approximately tree-like topology. Color encodes TDA-derived cluster
assignment, showing whether the geometric clusters are interpretable as real-world
entity groupings.

\begin{figure}[t]
  \centering
  \includegraphics[width=\linewidth]{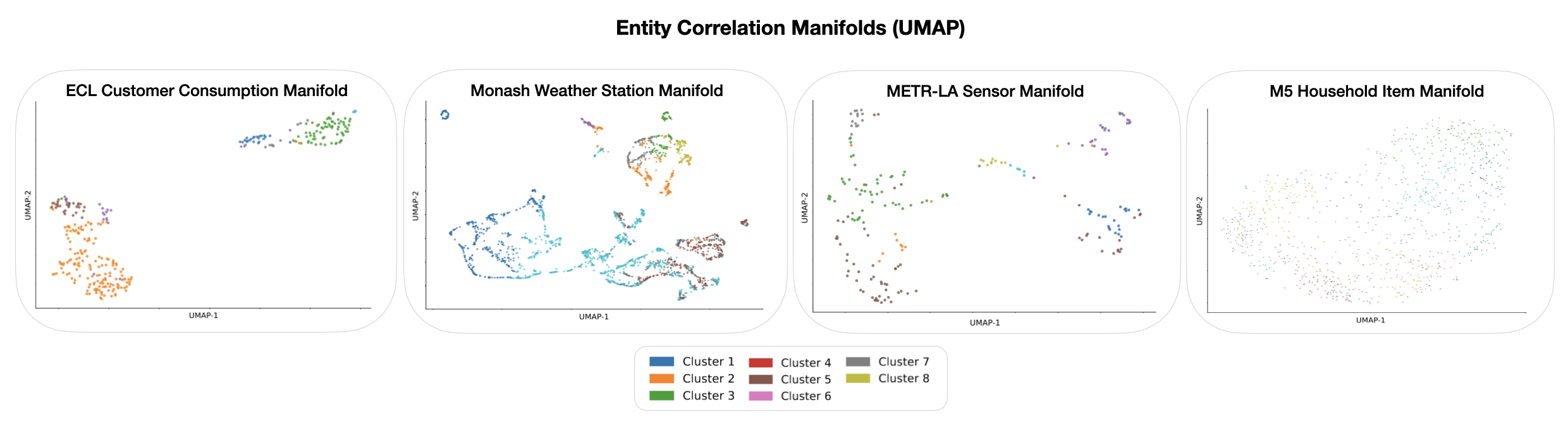}
  \caption{%
    \textbf{Entity correlation manifolds (UMAP) across four benchmarks.}
    Projection shape encodes topology: arc/loop structures indicate $H_1 > 0$
    (ECL, Weather); filamentary structure indicates near-tree topology with
    sparse local loops (METR-LA); diffuse cloud indicates low topological
    structure (M5 Household). Color encodes TDA-derived cluster assignment.
    M5 shows no visual clustering or geometric structure: a diffuse cloud
    with no arcs, loops, or filaments. This is consistent with its null
    pre-screening verdict (Table~\ref{tab:screening}), where calendar-dominated
    correlations produce no exploitable manifold geometry and no topology
    gain is expected or observed.
  }
  \label{fig:umap_manifolds}
\end{figure}

\subsection{Entity Cluster Profiles}
\label{app:cluster_profiles}

Figure~\ref{fig:cluster_profiles_all} shows the time-series profiles for each
TDA-derived cluster, confirming that the geometric groupings identified in the
UMAP projections correspond to interpretable real-world archetypes. For ECL and Weather, cluster
separation is sharp and semantically meaningful. For M5 and METR-LA, profiles are
more homogeneous, consistent with their calendar-dominated and near-tree manifold
verdicts respectively.

\begin{figure}[!ht]
  \centering
  \begin{subfigure}[t]{0.47\textwidth}
    \includegraphics[width=\linewidth]{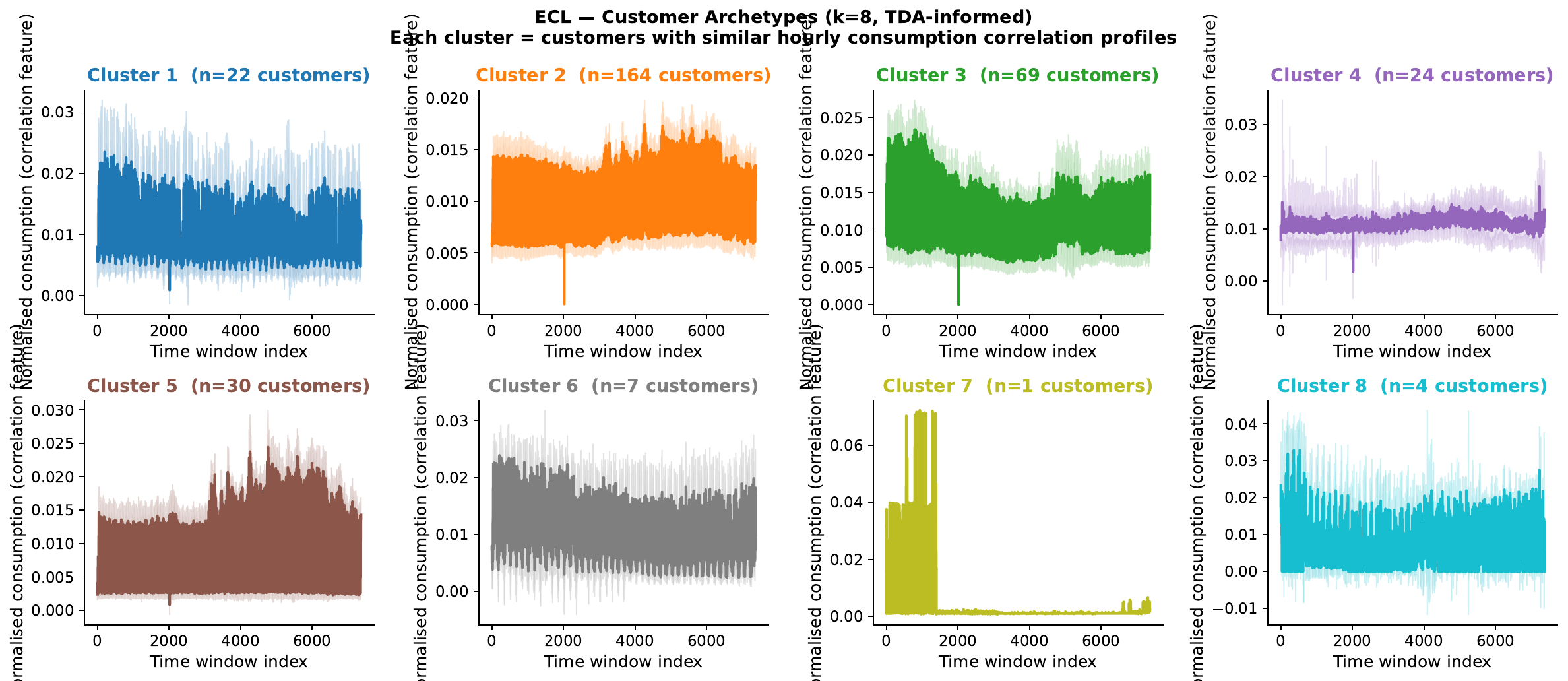}
    \caption{\textbf{ECL ($k{=}8$, TDA-informed).} Clusters correspond to distinct usage archetypes (high-daytime commercial, overnight industrial, flat residential, weekend-shifted), with Cluster~7 an outlier of extreme diurnal amplitude.}
    \label{fig:clusters_ecl}
  \end{subfigure}
  \hfill
  \begin{subfigure}[t]{0.47\textwidth}
    \includegraphics[width=\linewidth]{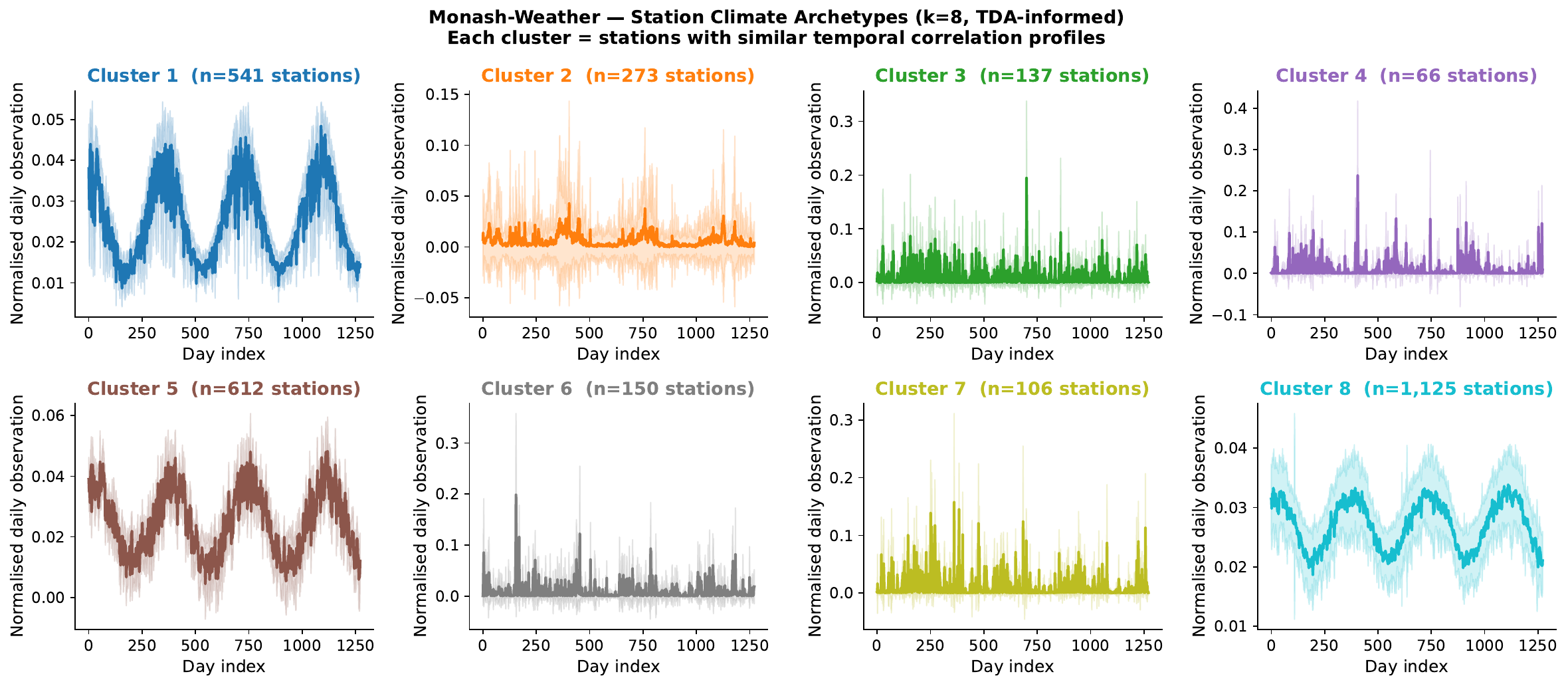}
    \caption{\textbf{Monash Weather.} Station clusters align with K\"{o}ppen climate classifications, confirming that TDA-derived groupings recover interpretable geographic climate structure.}
    \label{fig:clusters_weather}
  \end{subfigure}
  \\[3pt]
  \begin{subfigure}[t]{0.47\textwidth}
    \includegraphics[width=\linewidth]{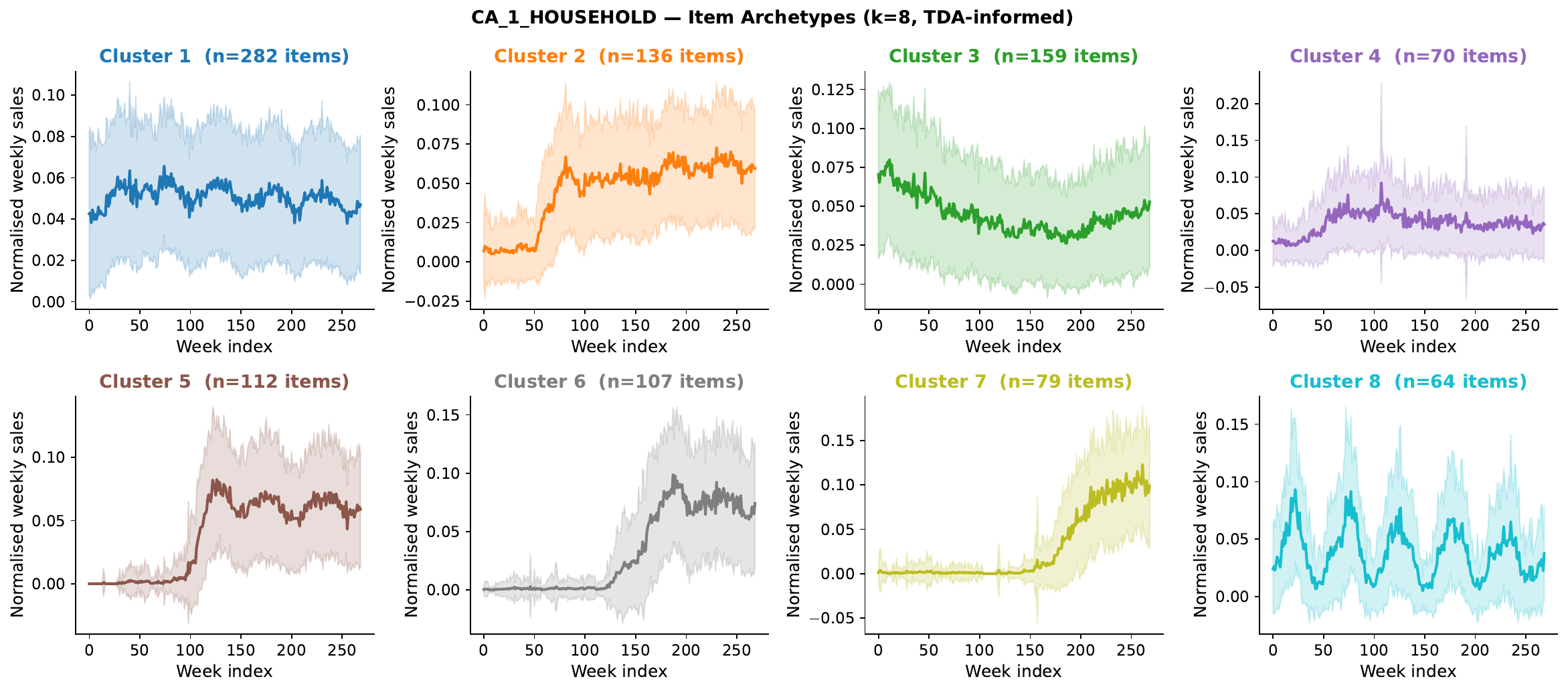}
    \caption{\textbf{M5 Household.} Item clusters correspond to demand archetypes (stable staple, seasonal, low-velocity, intermittent), with higher within-cluster variance than ECL or Weather consistent with weaker manifold structure.}
    \label{fig:clusters_m5}
  \end{subfigure}
  \hfill
  \begin{subfigure}[t]{0.47\textwidth}
    \includegraphics[width=\linewidth]{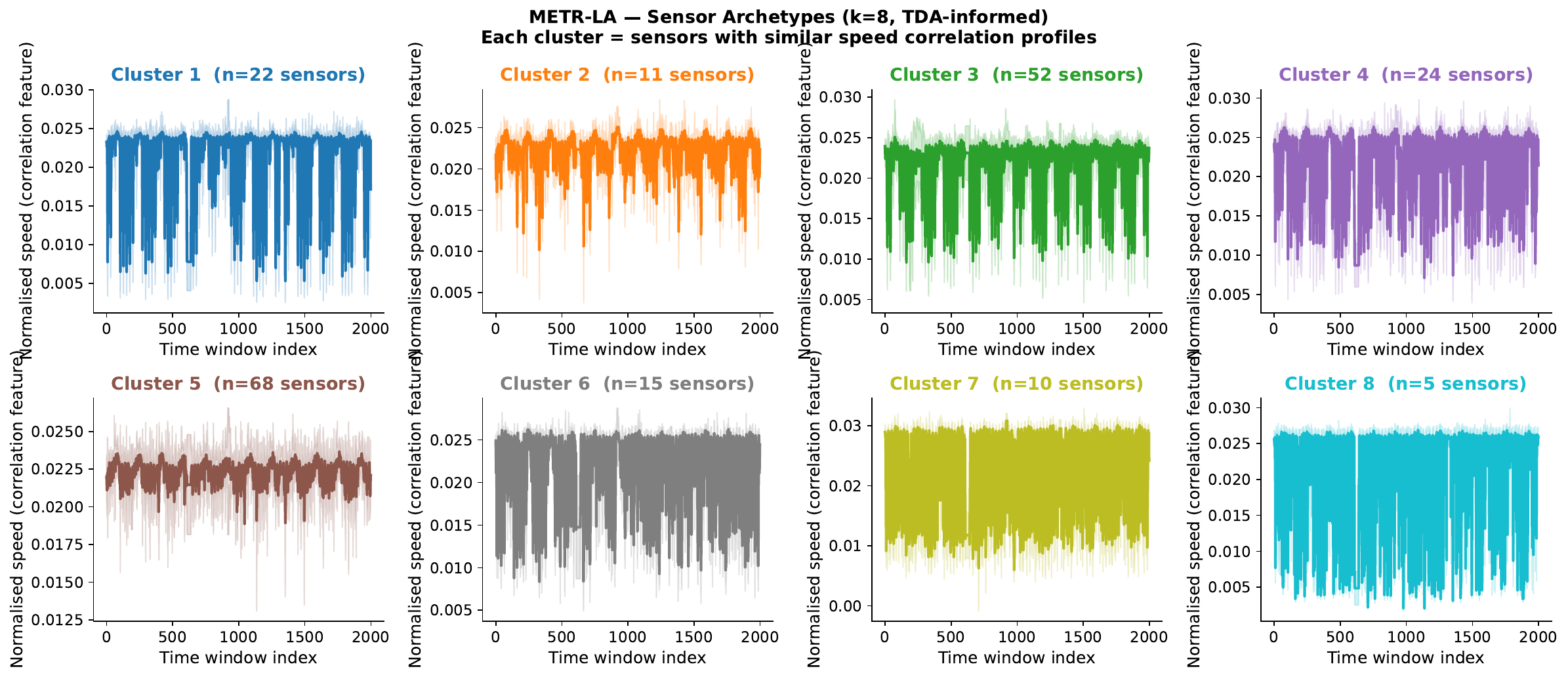}
    \caption{\textbf{METR-LA.} Sensor clusters correspond to highway archetypes (mainline freeway, interchange, ramp, arterial connector), with low within-cluster variance reflecting strong spatial regularity of freeway traffic patterns.}
    \label{fig:clusters_traffic}
  \end{subfigure}
  \caption{\textbf{Entity cluster profile visualizations across all four public
    benchmarks.} Each subfigure shows normalized time-series profiles per
    TDA-derived cluster (shaded band~$= \pm1\sigma$; solid line~= cluster mean).
    Cluster count $k$ is TDA-informed. ECL and Weather clusters are semantically
    sharp (usage/climate archetypes); M5 and METR-LA clusters are interpretable
    but less distinctive, consistent with their manifold verdicts.}
  \label{fig:cluster_profiles_all}
\end{figure}
\subsection{Mapper Graphs}
\label{app:mapper_graphs}

Figure~\ref{fig:mapper_all} shows Mapper graphs that summarize the global shape
of each dataset's population manifold, where nodes are clusters of entities with
similar time-series trajectories and edges connect clusters whose covers overlap.
$\beta_1$ (the first Betti number) counts the number of independent loops in the
manifold; nonzero $\beta_1$ indicates cyclic co-movement structure that TDA captures
as a diagnostic signal. Loop structures in the Mapper graph directly
corroborate nonzero $\beta_1$ from the persistence diagrams. Branching structures
reflect diverging sub-populations or regime transitions.

\begin{figure}[!ht]
  \centering
  \begin{subfigure}[t]{0.47\textwidth}
    \includegraphics[width=\linewidth]{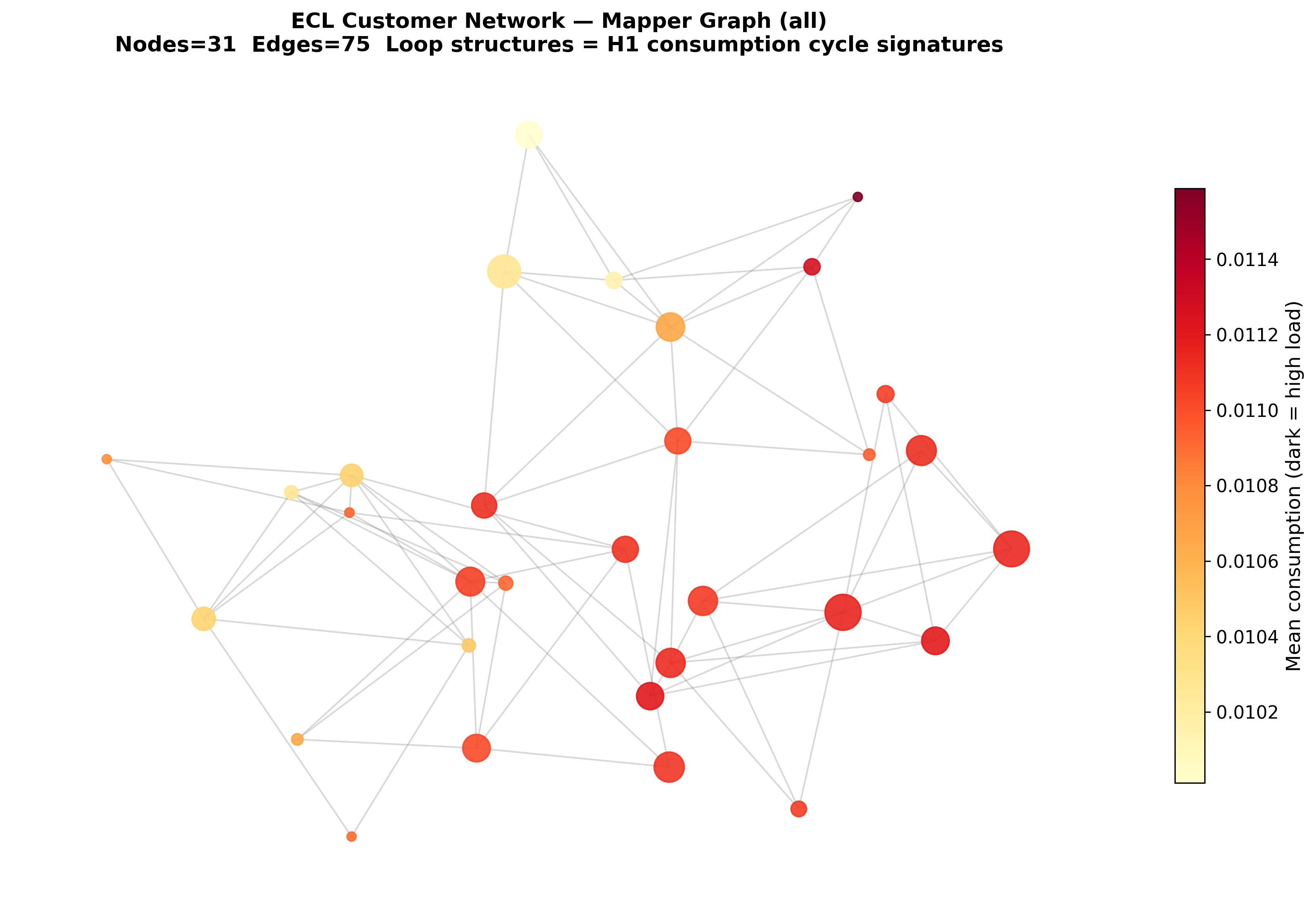}
    \caption{\textbf{ECL.} Pronounced loop structures link customers with similar but phase-shifted usage cycles, directly corroborating the nonzero $\beta_1$ from the persistence diagrams.}
    \label{fig:mapper_ecl}
  \end{subfigure}
  \hfill
  \begin{subfigure}[t]{0.47\textwidth}
    \includegraphics[width=\linewidth]{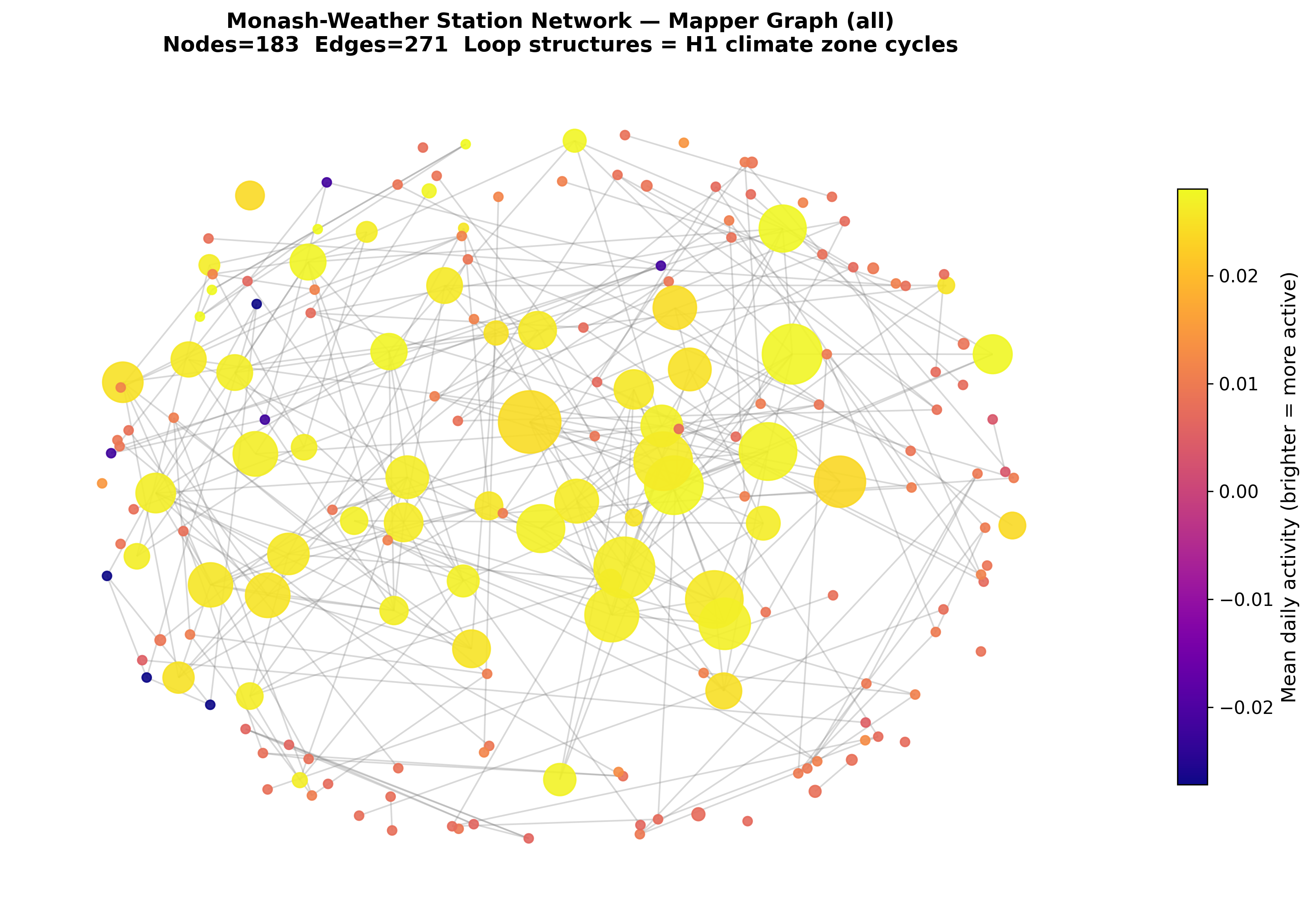}
    \caption{\textbf{Monash Weather.} Branching structures correspond to diverging climate subtypes and loop structures correspond to transitional regions connecting adjacent zones, consistent with the $\beta_1 > 0$ result.}
    \label{fig:mapper_weather}
  \end{subfigure}
  \\[3pt]
  \begin{subfigure}[t]{0.47\textwidth}
    \includegraphics[width=\linewidth]{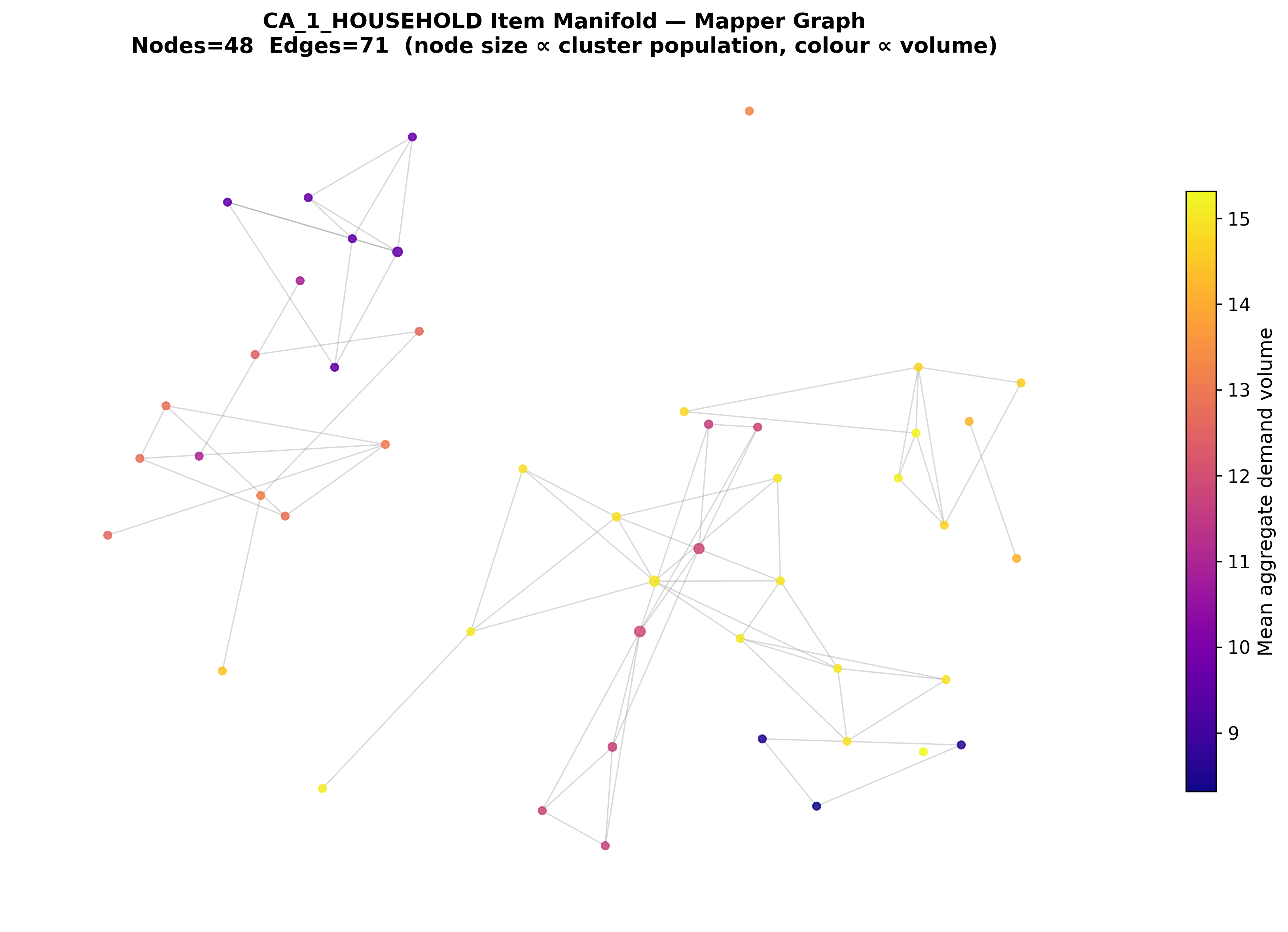}
    \caption{\textbf{M5 Household.} Loop structures in the graph identify product groups connected through shared demand periodicity (calendar co-movement), not genuine relational structure.}
    \label{fig:mapper_m5}
  \end{subfigure}
  \hfill
  \begin{subfigure}[t]{0.47\textwidth}
    \includegraphics[width=\linewidth]{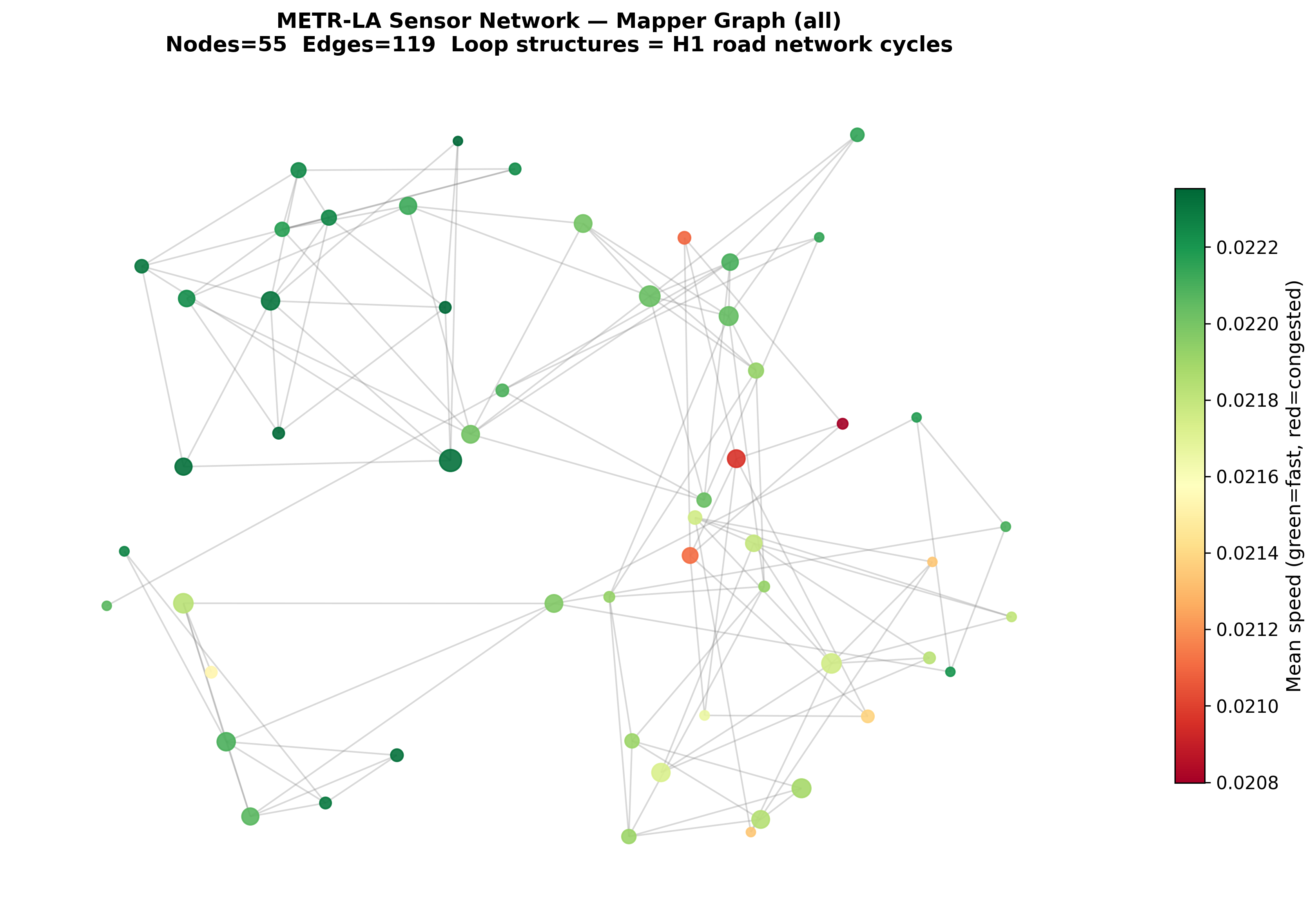}
    \caption{\textbf{METR-LA.} The graph topology mirrors the physical road network, with linear chains for highway segments and loops for interchanges and ring roads, directly validating the $\beta_1 > 0$ finding.}
    \label{fig:mapper_traffic}
  \end{subfigure}
  \caption{\textbf{Mapper graphs across all four public benchmarks.}
    Each node is a cluster of entities (customers, stations, items, or sensors)
    with similar trajectories; edges connect overlapping clusters. ECL and
    Weather show multi-loop structures consistent with high $\beta_1$; M5 loops
    are calendar-driven; METR-LA loops trace physical freeway interchanges.
  }
  \label{fig:mapper_all}
\end{figure}

\clearpage

\section{TDA Analysis: Internal Corpus}
\label{app:tda_internal}
\setlength{\floatsep}{4pt plus 1pt minus 1pt}
\setlength{\textfloatsep}{6pt plus 1pt minus 2pt}
\setlength{\abovecaptionskip}{4pt}
\setlength{\belowcaptionskip}{2pt}

This section characterizes the population manifold topology of the proprietary
corpus used in Sections~4.3--4.6, using the same analytical framework applied to
the public benchmarks in Appendix~\ref{app:homology}.
The corpus comprises five item categories (A-E) observed across four regions
(Regions~1--4), with each node-category pair treated as one population for TDA
fingerprinting.
The figures below use Category~A as the representative example throughout;
patterns for other categories are qualitatively consistent except where noted.

\subsection{Persistence Landscapes}
\label{app:tda_internal_landscapes}

Figure~\ref{fig:int_landscapes} overlays the persistence landscape curves for
$H_0$, $H_1$, and $H_2$ across all five categories.
$H_0$ activity is broadly similar across categories, reflecting comparable
cluster-merging dynamics at coarse scales.
The diagnostic signal resides in $H_1$. Categories~A and~D exhibit substantially
higher landscape amplitude than C and E, indicating denser recurring loop-like
structure across their item populations.
Category~C shows the lowest $H_1$ median persistence ($0.002$), suggesting a
more uniformly connected structural organization with fewer persistent cyclic
features.
The $H_2$ panel is active across all five categories, with the highest counts
in long-tail categories in Regions~1 and~3, where node and category fragmentation
creates persistent void structure that $H_2$ captures and that topology gains most
exploit.
These landscape signatures pass the pre-screening criterion of
Section~\ref{sec:screening}. The $H_1$ richness reflects genuine manifold structure, not
calendar artifact.

\begin{figure}[!ht]
  \centering
  \includegraphics[width=\linewidth]{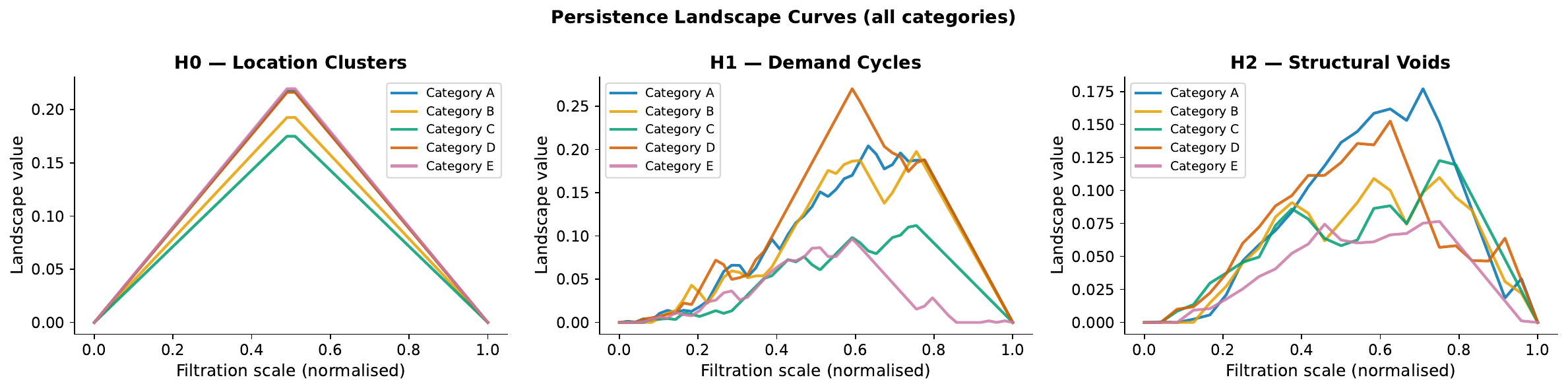}
  \caption{%
    \textbf{Persistence landscape curves $\lambda_1(t)$ for $H_0$, $H_1$, and $H_2$ across all five categories (internal corpus).}
    Categories~A and~D have substantially higher $H_1$ amplitude than C and E,
    indicating far more recurring loop-like structure in their item populations.%
  }
  \label{fig:int_landscapes}
\end{figure}

\subsection{Cross-Category Comparison}
\label{app:tda_internal_cross}

Figure~\ref{fig:int_cross_category} shows how TDA feature vectors differ across
categories, alongside $H_1$ Wasserstein-2 distance matrices quantifying
topological dissimilarity between them.
Categories~A and~D show meaningfully distinct topology from C and E, confirming
that the segments capture structurally different regimes.
Several category pairs exhibit pairwise Wasserstein distances of
$0.13$-$0.23$ across all regions, effectively sharing a single topological layer.
These cross-category coupling edges are reflected directly in the relational graph
($E_{\text{cross}}$).

\begin{figure}[!ht]
  \centering
  \includegraphics[width=\linewidth]{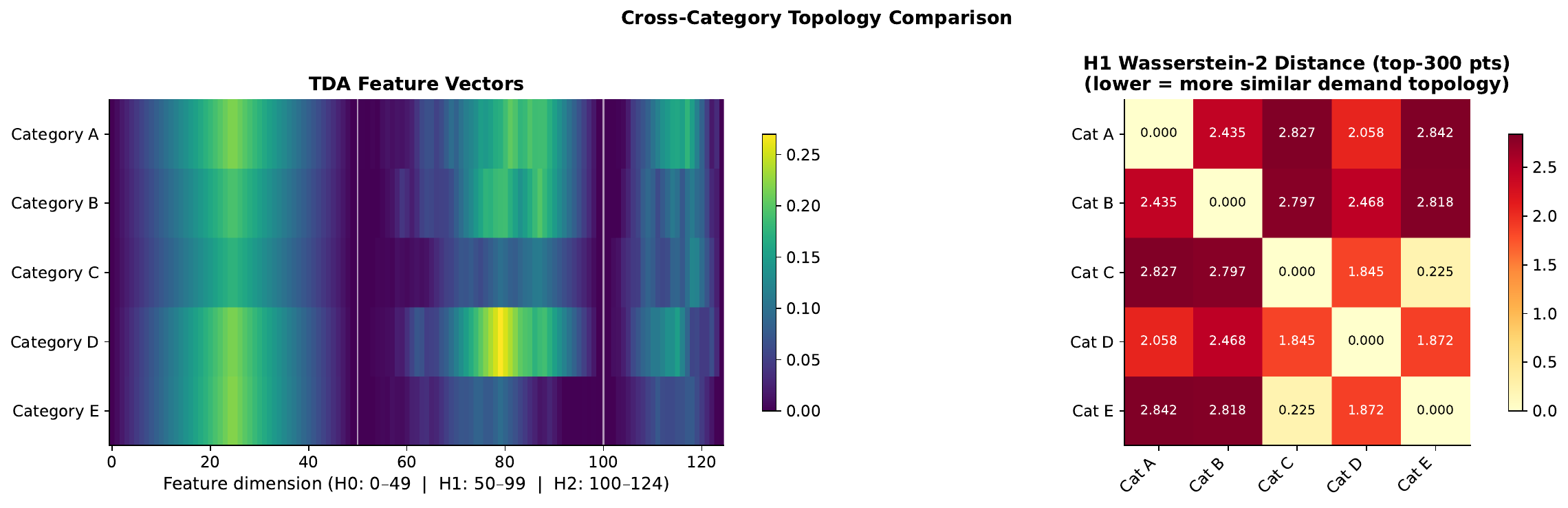}
  \caption{%
    \textbf{Cross-category TDA comparison (internal corpus).}
    Left panel: 125-dim TDA feature vectors per category ($H_0$: dims~0--49,
    $H_1$: dims~50--99, $H_2$: dims~100--124).
    Right panel: $H_1$ Wasserstein-2 pairwise distance matrix between categories
    (lower = more topologically similar).
    Categories~A and~D show structurally meaningful differences from C and E.%
  }
  \label{fig:int_cross_category}
\end{figure}
\subsection{Manifold Structure (UMAP Projection)}
\label{app:tda_internal_umap}

Figure~\ref{fig:int_umap} shows a UMAP 2D projection of Category~A item
trajectories, colored by TDA-derived cluster assignment.
The geometric coherence of cluster regions (compact, well-separated islands
rather than diffuse blobs) validates that the 125-dim persistence landscape
fingerprints capture genuine structural distinctions.
Items in peripheral UMAP regions are structural outliers whose isolation would
be collapsed into the majority by any purely volume-weighted representation.
This arc-and-island structure is the visual signature of $\beta_1 > 0$ and
confirms that Category~A passes the pre-screening criterion (Section~\ref{sec:screening}).

\begin{figure}[!ht]
  \centering
  \includegraphics[width=\linewidth]{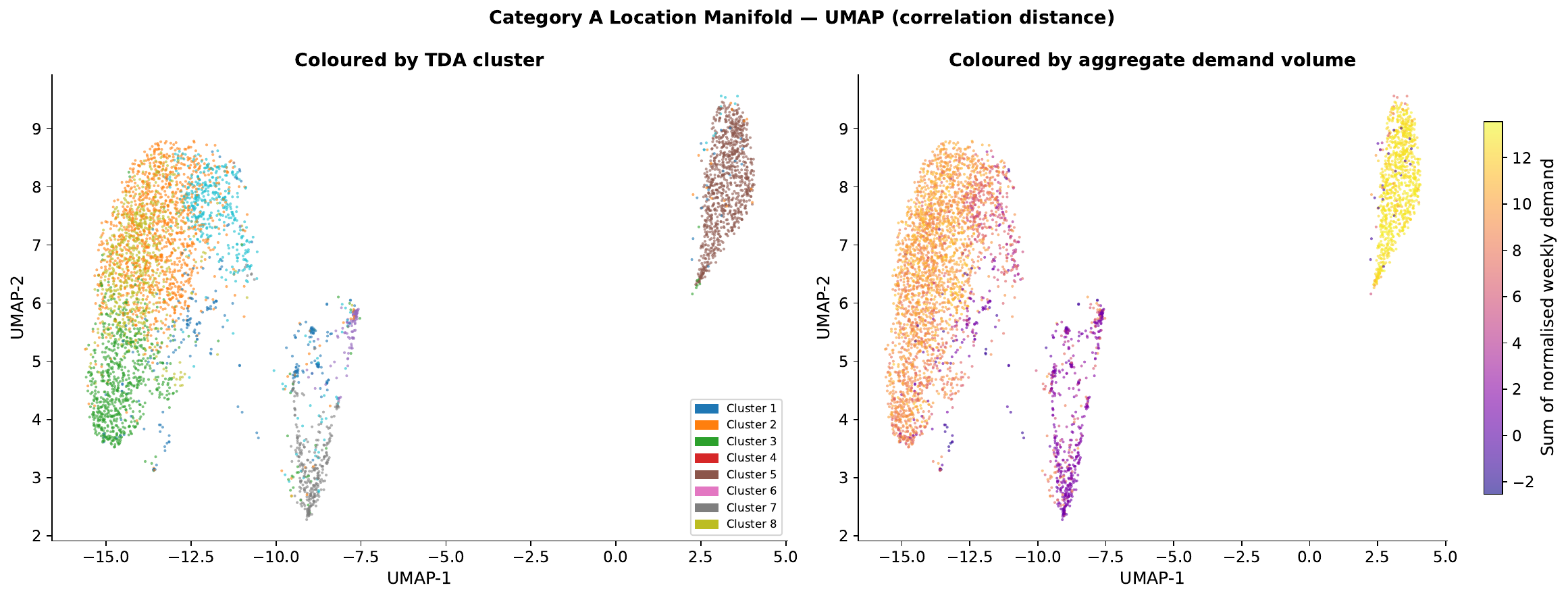}
  \caption{%
    \textbf{UMAP projection of Category~A item trajectories (internal corpus),
    colored by TDA-derived cluster.}
    Compact, well-separated cluster regions confirm that the persistence
    landscape fingerprints encode genuine structural distinctions.
    Peripheral items are structural outliers invisible to volume-weighted
    representations.%
  }
  \label{fig:int_umap}
\end{figure}

\subsection{Entity Cluster Profiles}
\label{app:tda_internal_clusters}

Figure~\ref{fig:int_clusters} shows item cluster assignments for Category~A
derived from K-means ($k{=}8$, TDA-informed) applied to the 125-dim TDA fingerprints.
Eight structurally distinct archetypes emerge from topology alone.
Two items can share nearly identical mean demand yet occupy different structural
clusters, representing fundamentally different temporal trajectory shapes that
volume-based assignment would collapse into the same group.
The cluster structure directly informs cold-start initialization: a newly introduced item inherits its topology signals from its nearest manifold neighbor within the same cluster, rather than from a volume-rank proxy.

\begin{figure}[!ht]
  \centering
  \includegraphics[width=\linewidth]{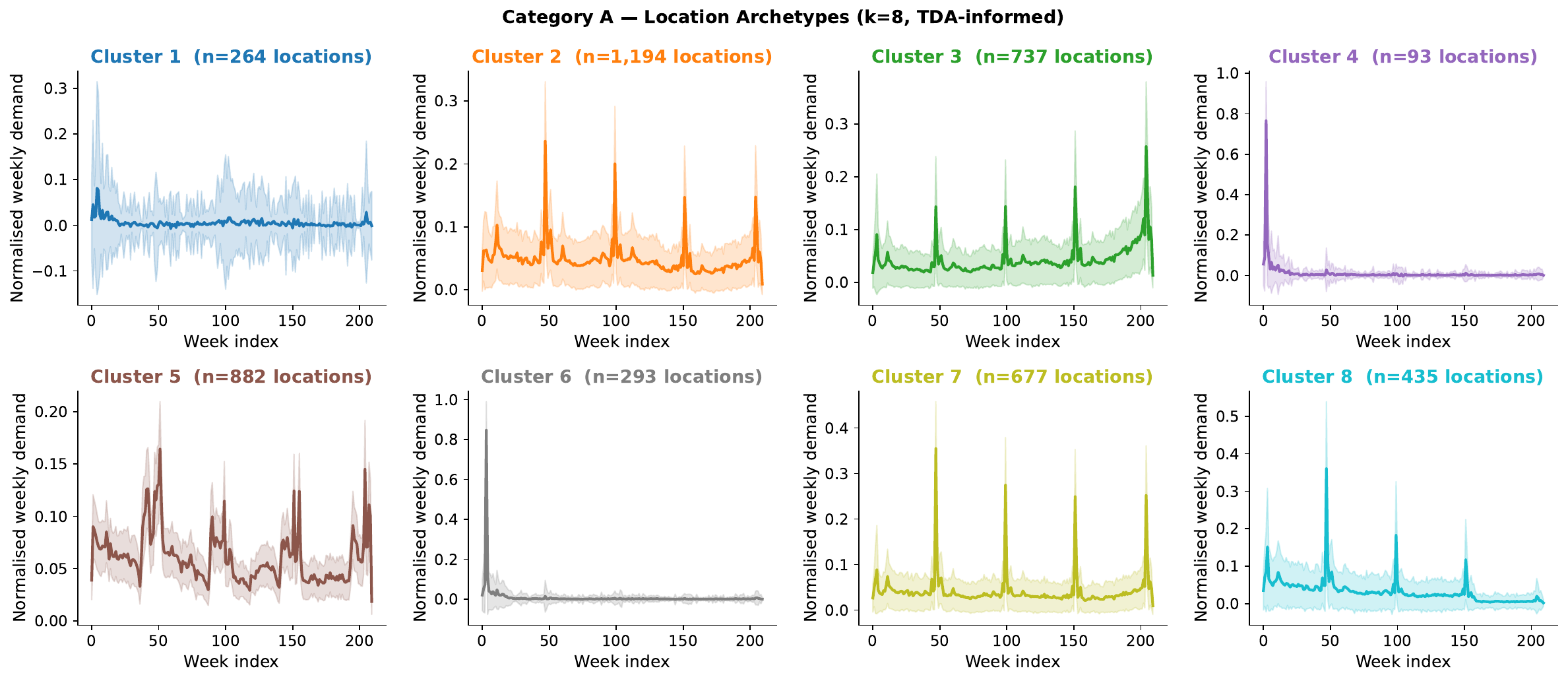}
  \caption{%
    \textbf{Item cluster assignments for Category~A (internal corpus),
    K-means $k{=}8$ (TDA-informed) on 125-dim TDA persistence landscape fingerprints.}
    Eight structurally distinct archetypes emerge from topology alone.
    Two items with nearly identical mean demand can occupy different clusters,
    representing fundamentally different forecast behaviors that volume-based
    assignment would collapse.%
  }
  \label{fig:int_clusters}
\end{figure}

\subsection{Mapper Graph}
\label{app:tda_internal_mapper}

Figure~\ref{fig:int_mapper} shows the Mapper graph for Category~A, computed
with a 2D UMAP lens, $12{\times}12$ cover with 40\% overlap, and DBSCAN
clustering within bins.
The graph has 55 nodes and reveals a hub-and-spoke macro-topology. There is a dense core
of structurally typical items and several peripheral branches of structural
outliers.
Loop structures in the graph directly corroborate the nonzero $\beta_1$ from the
persistence diagrams.
This macro-structure guides Layer~2 relational graph construction. Hub items
generalize well to each other and are strongly connected, while peripheral-branch
items connect only within their branch.

\begin{figure}[!ht]
  \centering
  \includegraphics[width=\linewidth]{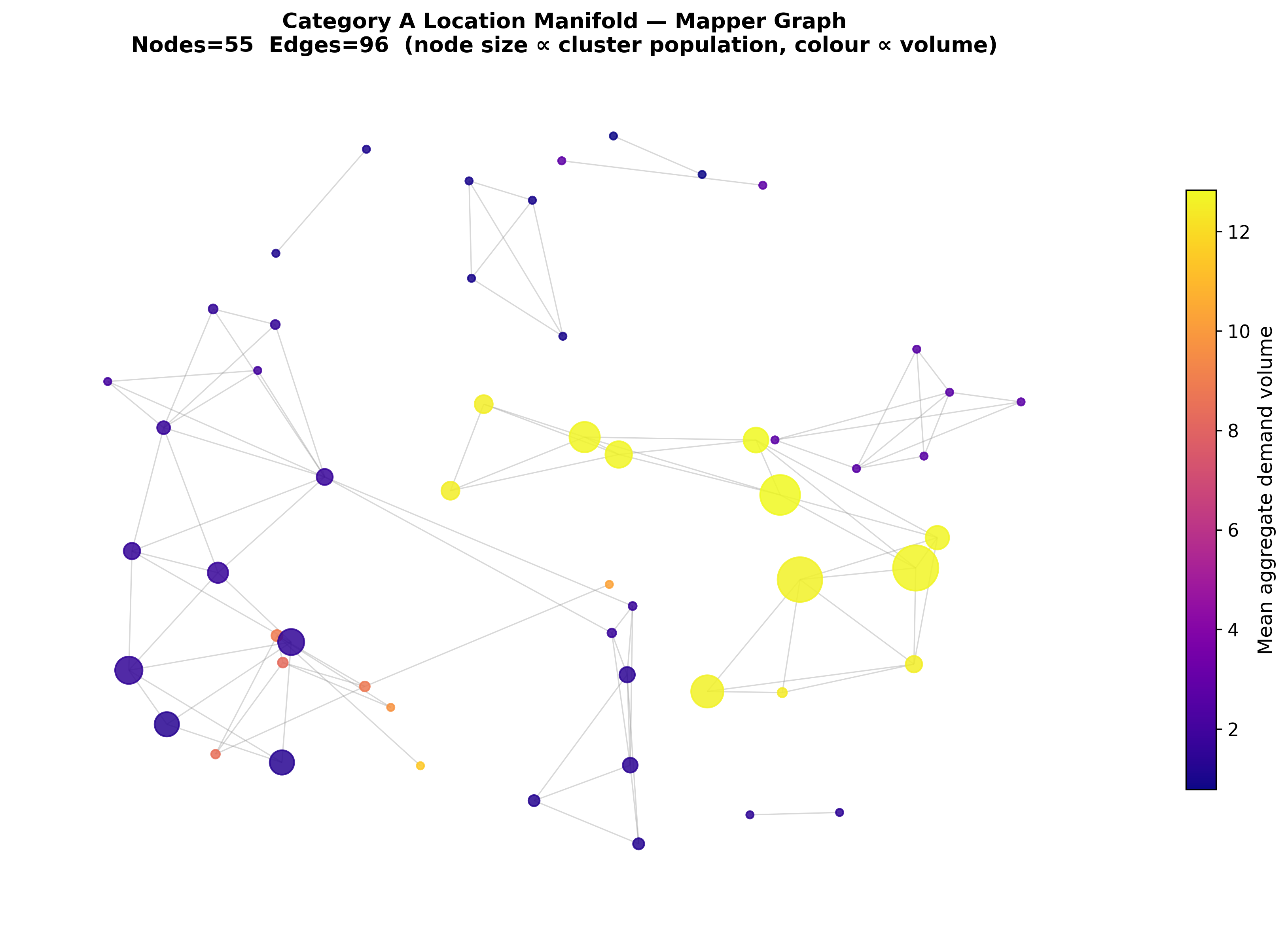}
  \caption{%
    \textbf{Mapper graph for Category~A (internal corpus).}
    55~nodes; node size $\propto$ item count; edge width $\propto$ inter-cluster
    affinity.
    The hub-and-spoke topology reveals a dense core of structurally typical items
    and several peripheral branches of structural outliers, directly corroborating
    the nonzero $\beta_1$ (first Betti number: count of independent loops in the manifold) from the persistence diagrams.%
  }
  \label{fig:int_mapper}
\end{figure}

\clearpage

\section{Randomized Control Ablations}
\label{app:ablations}

To verify that gains reflect genuine manifold structure rather than the effect of
adding any non-zero context, we evaluate two controls. \textsc{rand\_TDA} replaces
the TDA vector with Gaussian noise of matched dimension. \textsc{shuffle\_TDA}
uses real TDA vectors permuted randomly across series, preserving marginal
statistics while breaking the series-to-topology correspondence.
Table~\ref{tab:ablation-controls} summarizes results on the two public benchmarks
where both controls were evaluated. These two benchmarks represent opposite ends
of the topology spectrum: Monash Weather has the highest $\beta_1$ count among
public benchmarks ($H_1 = 1{,}847$), making topology-driven gains most
pronounced, while METR-LA has the sparsest topology (near-tree structure,
$H_1 \approx 0$), providing a stringent null test where genuine topology gains
are expected to be small.

\begin{table}[h]
\centering\small
\caption{\textbf{Randomized control ablations.}
  Monash Weather: MAE at $H{=}30$; METR-LA: MAE at 30\,min (mid-range horizon where control separation is most stable).
  $\Delta$ relative to Fixed~[no topo]. $\downarrow$ lower is better.}
\label{tab:ablation-controls}
\setlength{\tabcolsep}{6pt}
\begin{tabular}{l r@{~}l r@{~}l r@{~}l}
\toprule
\textbf{Model}
  & \multicolumn{2}{c}{\textbf{Monash Weather MAE}}
  & \multicolumn{2}{c}{\textbf{METR-LA MAE}} \\
\midrule
Transformer
  & 2.175 &                                                          & 1.540 & \\
Transformer + rand-TDA \textit{(control)}
  & 2.182 & \small{$(+0.007)$}         & 1.574 & \small{$(+0.034)$} \\
Transformer + shuffle-TDA \textit{(control)}
  & 2.199 & \small{$(+0.024)$}                  & 1.535 & \small{$(-0.005)$} \\
Transformer + TDA
  & 2.170 & \small{$(-0.005)$}                   & 1.540 & \small{$(0.000)$} \\
Transformer $+$ TDA $+$ Sheaf
  & \textbf{2.004} & \small{$(-0.171)$}             & \textbf{1.521} & \small{$(-0.019)$} \\
\bottomrule
\end{tabular}
\smallskip\\
\end{table}

On Monash Weather, the topology-richest benchmark ($1{,}847$ $H_1$ generators),
random injection actively regresses $(+0.007)$ while real TDA improves $(-0.005)$
and Sheaf achieves the largest gain $(-0.171)$. Shuffle performance falls
between rand and real on both datasets, confirming that both the content of the
TDA vector \emph{and} its correct series-to-topology assignment contribute to
observed improvements.

\section{Spectral vs.\ Neural Sheaf Encoder}
\label{app:spectral_vs_neural}

As described in Section~\ref{sec:sheaf}, we evaluate two implementations of the
sheaf component.
\textbf{Sheaf (Spectral)} computes per-series coordinates as the leading left
singular vectors of the entity-time training matrix, block-wise per
store$\times$category for M5 and globally for ECL, Monash Weather, and METR-LA.
\textbf{Sheaf (Neural)} trains per-series learnable embeddings $E_i \in
\mathbb{R}^{256}$ using a coboundary consistency loss on a $k$-NN correlation
graph, warm-started from the spectral coordinates (Section~\ref{sec:sheaf}).
The objective is
  \[
    \mathcal{L}_{\text{sheaf}} =
    \lambda_c \!\!\sum_{(i,j)\in\mathcal{E}}\!\!
  w_{ij}\,\bigl\|R(E_i)-R(E_j)\bigr\|^2
    \;+\;
    \lambda_r \,\bigl\|\mathrm{dec}(E) -
  \mathbf{x}_{\text{node}}\bigr\|^2
    \;-\; \beta\,\operatorname{Var}_n(R(E_n)),
  \]
  where $R$ is a shared linear restriction map, the first term
  is the sheaf
  coboundary loss, the second is a reconstruction regularizer,
  and the third is a
  spread penalty that prevents trivial collapse to the zero
  section.

Table~\ref{tab:spectral_vs_neural} reports the comparison on M5, the only dataset
where both variants were fully evaluated across all backbone families.
The spectral encoder consistently matches or outperforms the neural variant.
The most informative comparison is Chronos 2.0 on Household, where individual
item demand is most heterogeneous. There, spectral coordinates achieve MAE~1.0251 versus
1.0343 for the neural encoder, a 0.9\% relative gap.
Across the full 28{,}860-series evaluation the ordering is
Vanilla~(0.7717)~$<$~Spectral~(0.7742)~$<$~Neural~(0.7805),
reflecting that M5 is a null-TDA dataset and neither sheaf variant can overcome the
absent global topology signal. However, the spectral variant consistently degrades
less than the neural variant from the vanilla baseline.

The neural encoder degrades relative to spectral for three compounding reasons.
First, warm-starting from the spectral coordinates and then applying gradient
descent toward coboundary consistency moves embeddings away from the spectral
position toward graph-agreement, which is less useful for downstream forecasting.
Second, the coboundary loss and reconstruction loss work in opposition. Pushing
two correlated items toward agreement in restriction-map space reduces their
  individual reconstruction quality.
Third, the spectral encoder extracts the dominant structural axes of
  entity co-variation via closed-form truncated SVD, requiring no
  training. The neural encoder approximates
a more expressive but empirically inferior objective.
We therefore adopt spectral coordinates as the default for all reported results.

\begin{table}[h]
\centering
\caption{%
  \textbf{Spectral vs.\ neural sheaf encoder on M5 (Walmart, 28{,}860 active series).}
  MAE and WAPE per M5 category. \textbf{Bold} = best per section. $\downarrow$ lower is better.
  Household has the most heterogeneous item demand and the most informative comparison point for the two sheaf
  variants.
  All rows use a 52-week context window and 4-week forecast horizon.%
}
\label{tab:spectral_vs_neural}
\small\setlength{\tabcolsep}{6pt}
\begin{tabular}{lcccccc}
\toprule
 & \multicolumn{2}{c}{\textbf{Foods}} & \multicolumn{2}{c}{\textbf{Hobbies}} & \multicolumn{2}{c}{\textbf{Household}} \\
\cmidrule(lr){2-3}\cmidrule(lr){4-5}\cmidrule(lr){6-7}
\textbf{Model}
  & \textbf{MAE$\downarrow$} & \textbf{WAPE$\downarrow$}
  & \textbf{MAE$\downarrow$} & \textbf{WAPE$\downarrow$}
  & \textbf{MAE$\downarrow$} & \textbf{WAPE$\downarrow$} \\
\midrule
\multicolumn{7}{l}{\textit{Chronos 2.0 variants}} \\
Chronos Vanilla Adapter      & 0.984 & 1.408 & 1.054 & 1.788 & 1.040 & 1.643 \\
Chronos $+$ TDA              & 0.982 & 1.405 & 1.052 & 1.786 & 1.039 & 1.641 \\
Chronos $+$ TDA $+$ Sheaf    & \textbf{0.966} & \textbf{1.383} & \textbf{1.038} & \textbf{1.761} & \textbf{1.025} & \textbf{1.618} \\
Chronos $+$ TDA $+$ Sheaf (Neural)   & 0.981 & 1.403 & 1.043 & 1.770 & 1.034 & 1.633 \\
\bottomrule
\end{tabular}
\end{table}

Spectral coordinates dominate the neural encoder on every evaluated configuration.
The gap is largest where item-level heterogeneity is highest (Household) and
smallest where the overall manifold signal is weakest (full M5 evaluation, a
null-TDA domain). These results support adopting the spectral sheaf encoder as the
appropriate default for TopoPrimer across domains.

\paragraph{Why M5 is the right comparison domain.}
M5 is a null-TDA domain. Its $H_1$ generators reflect shared seasonal periodicity
rather than genuine relational structure, and the pre-screening verdict is null
(Table~\ref{tab:screening}). This makes it the cleanest isolated test of the
sheaf encoder itself. On ECL and Monash Weather, both TDA and Sheaf are active simultaneously; any difference between sheaf variants is confounded by the concurrent TDA contribution, making it impossible to isolate the sheaf's effect alone. On M5, TDA contributes nothing, so the sheaf coordinate must stand alone and any quality difference between SVD and the neural encoder is fully unmasked. Chronos~2.0 is the most informative backbone here: unlike TimesFM, which was pre-trained on M5, Chronos had no M5 exposure, so the sheaf coordinate carries genuine signal and the spectral advantage is cleanly observable (Household MAE 1.0251 vs.\ 1.0343). The fully-trained Transformer, which uses spectral coordinates throughout, is not included as a separate row.

\paragraph{Computational overhead.}
The performance advantage of spectral coordinates is reinforced by a decisive
cost asymmetry. The spectral encoder requires a single truncated SVD of the
entity-time training matrix, block-wise per store$\times$category group.
For M5 at 30{,}490 series this completes in under 90 seconds on a single CPU
core and requires no hyperparameter selection.
The neural encoder requires a full pre-training run. It needs multiple epochs of gradient
descent on a coboundary consistency loss with three hyperparameters
($\lambda_c$, $\lambda_r$, $\beta$), early stopping on a validation criterion,
and warm-start initialization from the spectral coordinates themselves.
TimesFM was not evaluated with the neural encoder; given
  that spectral coordinates outperform neural on Chronos~2.0 on the
  same domain, the additional pre-training cost is not warranted.
When performance is equivalent or spectral coordinates are superior, the
zero-training-cost encoder is the unambiguous choice.



\newpage
\section{Topology Signal Survives Fine-Tuning}
\label{app:ft_robustness}

A natural objection to topology augmentation for adapted models is that
fine-tuning on in-domain data should subsume any topological signal.
The result is inconsistent with that hypothesis.

\paragraph{Why open benchmarks do not suffice.}
A meaningful test requires two conditions. First, the fine-tuning pass must give the
backbone access to \emph{relational} domain structure, so gradient descent has a
mechanism to subsume cross-series topology. Second, the dataset must be
topology-rich enough that the topology gain within the Fine-Tuned family is measurable above noise.

ECL satisfies the topology-richness condition with $83$ $H_1$ cycles from shared grid infrastructure but
fails the relational-structure condition. The 321 clients are anonymous meters with no entity graph and no
relational labels. Fine-tuning on ECL produces a checkpoint that differs from the
zero-shot one in distributional calibration, not in cross-series relational signal.
Persistent homology captures structure that the fine-tuning objective never
targeted, so comparing fine-tuned adapters with and without topology does not test whether
fine-tuning subsumes topology. The backbone never had a mechanism to do so.

Monash Weather fails for a complementary reason. Its manifold is topology-rich
($1{,}847$ $H_1$ generators), but Chronos was pre-trained on Monash. The Zero-Shot-vs.-Fine-Tuned
gap is negligible, making the within-Fine-Tuned comparison vacuous for the primary backbone.
For any other backbone, the entity-graph-free issue from ECL resurfaces.
METR-LA and M5 have fine-tunable structures but near-null topology signal.
We therefore evaluate on the internal corpus from
Section~\ref{sec:three_regimes}, which provides both conditions simultaneously.

\paragraph{Setup.}
We evaluate two Chronos~2.0 adapter families on the internal corpus. The fine-tuning robustness question requires a pre-trained foundation model with a meaningful zero-shot baseline; the fully-trained Transformer does not have one, so the Transformer family is used for the seasonal-spike and cold-start evaluations instead.
\textbf{Chronos Zero-Shot} trains an adapter head on the frozen zero-shot Chronos
checkpoint. \textbf{Chronos Fine-Tuned} trains an adapter head on a Chronos checkpoint
fine-tuned on domain data. Within each family we evaluate three adapter
configurations: \textbf{vanilla} (no topology), \textbf{$+$TDA$_E$} (125-dim entity-manifold
persistence landscape), and \textbf{$+$TDA$_E$$+$TDA$_I$} (entity-manifold plus the additional 125-dim item-manifold fingerprint). Sheaf coordinates are omitted here because the fine-tuning robustness question targets topology fingerprints specifically, which are the components most plausibly subsumed by gradient descent on in-domain data; sheaf results on these backbone families appear in the main results (Table~\ref{tab:toposheaf_public}).
Table~\ref{tab:ft_robustness} reports MAE and WAPE on a single-category
slice ($N=50{,}920$ series) with per-family $\Delta$ values.

\paragraph{Backbone fine-tuning provides marginal gain over the vanilla adapter.}
On the single-category slice, Chronos Zero-Shot achieves MAE~$1.168$ and
Chronos Fine-Tuned achieves MAE~$1.142$, a marginal improvement of $-0.026$ MAE from fine-tuning the backbone. This is consistent with the general
observation that foundation model adapter performance on sparse discontinuous
demand is difficult to improve via standard fine-tuning alone.

\paragraph{Topology gain is preserved after fine-tuning.}
On the single-category slice (Table~\ref{tab:ft_robustness}), the topology gain within the Zero-Shot family
is $\Delta\mathrm{MAE} = -0.022$ ($\Delta$WAPE~$-0.016$) and
the topology gain within the Fine-Tuned family is $\Delta\mathrm{MAE} = -0.024$
($\Delta$WAPE~$-0.017$). These deltas are essentially
identical across backbone conditions that differ substantially in domain
adaptation, consistent with the hypothesis that fine-tuning and topology
augmentation capture largely orthogonal signals.
Chronos Fine-Tuned $+$ TDA$_E$$+$TDA$_I$ achieves the best result across both families
(MAE~$1.118$, WAPE~$0.861$). The combined
gain from fine-tuned initialization and topology over Chronos Zero-Shot
is $-0.055$ WAPE, and the two contributions are additive.
One nuance is visible in Table~\ref{tab:ft_robustness}: on the fine-tuned backbone,
TDA$_E$ alone yields only a marginal MAE improvement ($-0.005$) and a slight WAPE
regression ($+0.009$) relative to Chronos Fine-Tuned vanilla.
The node-manifold fingerprint provides insufficient positional context once the backbone
has already been exposed to domain structure via fine-tuning; the item-manifold
fingerprint TDA$_I$ is required to recover and exceed the vanilla baseline.
The topology gain within the Fine-Tuned family of $\Delta = -0.028$ WAPE reported above refers specifically
to Chronos Fine-Tuned $+$ TDA$_E$$+$TDA$_I$, not to TDA$_E$ in isolation.

\paragraph{Implication.}
These results suggest that fine-tuning and topology augmentation address complementary aspects of the
forecasting problem, with largely additive benefits. Fine-tuning improves
backbone calibration to the domain distribution; topology supplies
cross-series structural information that the univariate fine-tuning objective
has no mechanism to recover. This motivates TopoPrimer as a persistent,
complementary input representation compatible with any level of backbone
adaptation.

\begin{table}[h]
\centering
\caption{%
  \textbf{Chronos Zero-Shot vs.\ Chronos Fine-Tuned adapter families} on a single-category internal
  corpus slice ($N=50{,}920$ series). $\Delta$ computed relative to vanilla
  adapter within each backbone family. Gain magnitude is near-identical across
  both backbone conditions despite substantial differences in domain adaptation.
  \textbf{Bold} = best per section. $\downarrow$ lower is better.
}
\label{tab:ft_robustness}
\begin{tabular}{lcccc}
\toprule
\textbf{Model} & \textbf{MAE$\downarrow$} & \textbf{WAPE$\downarrow$}
  & \textbf{$\Delta$MAE} & \textbf{$\Delta$WAPE} \\
\midrule
\multicolumn{5}{l}{\textit{Chronos 2.0 Zero-Shot variants}} \\
Chronos Zero-Shot                    & 1.168 & 0.849 & --          & --          \\
Chronos Zero-Shot $+$ TDA$_E$        & 1.146 & 0.833 & \dneg{0.022} & \dneg{0.016} \\
\midrule
\multicolumn{5}{l}{\textit{Chronos 2.0 Fine-Tuned variants}} \\
Chronos Fine-Tuned                   & 1.142 & 0.878 & --          & --          \\
Chronos Fine-Tuned $+$ TDA$_E$       & 1.137 & 0.887 & \dneg{0.005} & $+0.009$ \\
Chronos Fine-Tuned $+$ TDA$_E$$+$TDA$_I$ & \best{1.118} & \best{0.861} & \dneg{0.024} & \dneg{0.017} \\
\bottomrule
\end{tabular}
\end{table}

\section{Quantile Calibration on the Internal Corpus}
\label{app:calibration}

The Transformer backbone outputs 9 calibrated quantiles (0.02, 0.1, 0.2, 0.3, 0.5, 0.7, 0.8, 0.9, 0.98) via a Huber quantile loss. Table~\ref{tab:qloss} reports average pinball loss (QLoss) across all quantiles on the internal corpus, alongside the corresponding MAE.

On the internal corpus, where $H_1/N$ is non-negligible and the TDA fingerprint carries genuine manifold signal (unlike the public benchmarks, where TDA alone was consistently flat or slightly negative), both topology components improve substantially over the vanilla baseline. A meaningful dissociation between the two emerges. \textbf{Transformer~$+$~TDA$_E$$+$TDA$_I$ achieves the best MAE~(0.596)}: the combined TDA fingerprints inject population-level structural signal that lowers the median forecast. \textbf{Transformer~$+$~Sheaf achieves the best QLoss~(0.1675)}: the sheaf coordinate smooths the full predictive distribution and improves tail coverage even when it does not push the median lower. On the internal corpus, TDA$_E$$+$TDA$_I$ improves point accuracy and sheaf improves calibration; the two contributions are complementary. This is consistent with the orthogonal-signal interpretation in Section~\ref{sec:three_regimes}.

\begin{table}[h]
\centering
\caption{%
  \textbf{Quantile calibration on the internal corpus (Transformer backbone).}
  QLoss = average pinball loss across 9 quantiles
  (0.02, 0.1, 0.2, 0.3, 0.5, 0.7, 0.8, 0.9, 0.98), on z-normalized outputs.
  MAE is on actual-scale outputs.
  \textbf{Bold} = best per section. $\downarrow$ lower is better.%
}
\label{tab:qloss}
\small\setlength{\tabcolsep}{12pt}
\begin{tabular}{lcc}
\toprule
\textbf{Model} & \textbf{QLoss}$\downarrow$ & \textbf{MAE}$\downarrow$ \\
\midrule
Transformer                               & 0.2637 & 0.802 \\
Transformer $+$ TDA$_E$                  & 0.2224 & 0.692 \\
Transformer $+$ TDA$_E$ $+$ TDA$_I$     & 0.1687 & \textbf{0.596} \\
Transformer $+$ Sheaf                    & \textbf{0.1675} & 0.629 \\
\bottomrule
\end{tabular}
\end{table}

\newpage
\section{Internal Corpus: Cold-Start and Seasonality MAE Tables}
\label{app:internal_coldstart}

Sections~\ref{sec:seasonal_spikes} and~\ref{sec:coldstart} describe the
seasonality and cold-start evaluations on the internal corpus.
The main text reports summary statistics and figure panels;
this appendix provides the complete per-week MAE tables.

  \begin{table}[h]
  \centering
  \caption{Seasonality spikes MAE over peak-demand window (all
  items). DLinear and NLinear from~\citet{zeng2023transformers}.}
  \label{tab:seasonality_mae}
  \begin{tabular}{lcccc}
  \toprule
  \textbf{Model} & \textbf{Week 0} & \textbf{Week 1} & \textbf{Week 2} & \textbf{Week 3} \\
  \midrule
  \multicolumn{5}{l}{\textit{Transformer variants}} \\
  Transformer                                      & 2.016 & 2.176 & 2.228 & 2.250 \\
  Transformer $+$ TDA$_E$                          & \textbf{1.725} & 1.929 & 2.023 &
   2.002 \\
  Transformer $+$ TDA$_E$ $+$ TDA$_I$              & 2.020 & 2.121 & 2.031 & 2.112 \\
  Transformer $+$ TDA$_E$ $+$ TDA$_I$ $+$ Sheaf   & 1.781 & \textbf{1.909} &
  \textbf{1.872} & \textbf{1.924} \\
  \midrule
  \multicolumn{5}{l}{\textit{Classical baselines}} \\
  Rate-based  & \textbf{1.985} & \textbf{2.191} & \textbf{2.387} & \textbf{2.874} \\
  DLinear     & 2.089 & 2.339 & 2.518 & 3.060 \\
  NLinear     & 1.942 & 2.136 & 2.293 & 2.757 \\
  XGBoost     & 2.272 & 2.519 & 2.743 & 3.368 \\
  \midrule
  \multicolumn{5}{l}{\textit{Zero-shot TSFMs}} \\
  Chronos     & \textbf{1.853} & \textbf{2.049} & \textbf{2.219} & \textbf{2.780} \\
  TimesFM     & 2.082 & 2.300 & 2.504 & 2.981 \\
  \bottomrule
  \end{tabular}
  \end{table}

\begin{table}[h]
  \centering
  \caption{Cold-start MAE vs.\ weeks of post-launch history (new
  items only).}
  \label{tab:cold_start_mae}
  \begin{tabular}{lcccc}
  \toprule
  \textbf{Model} & \textbf{Week 0} & \textbf{Week 1} & \textbf{Week 2} & \textbf{Week 3} \\
  \midrule
  \multicolumn{5}{l}{\textit{Transformer variants}} \\
  Transformer                                      & 1.887 & 1.690 & 1.565 & 1.535 \\
  Transformer $+$ TDA$_E$                          & 1.733 & 1.555 & 1.412 &
  \textbf{1.353} \\
  Transformer $+$ TDA$_E$ $+$ TDA$_I$              & \textbf{1.375} & 1.385 & 1.380 &
   1.458 \\
  Transformer $+$ TDA$_E$ $+$ TDA$_I$ $+$ Sheaf   & 1.395 & \textbf{1.388} &
  \textbf{1.385} & 1.524 \\
  \midrule
  \multicolumn{5}{l}{\textit{Classical baselines}} \\
  Rate-based  & 1.796 & 1.716 & \textbf{1.497} & \textbf{1.525} \\
  DLinear     & 1.788 & 1.716 & 1.646 & 1.583 \\
  NLinear     & \textbf{1.700} & \textbf{1.617} & 1.652 & 1.599 \\
  XGBoost     & 1.921 & 1.866 & 1.779 & 1.733 \\
  \midrule
  \multicolumn{5}{l}{\textit{Zero-shot TSFMs}} \\
  Chronos     & \textbf{1.557} & \textbf{1.557} & 1.550 & 1.538 \\
  TimesFM     & 1.946 & 1.672 & \textbf{1.500} & \textbf{1.519} \\
  \bottomrule
  \end{tabular}
  \end{table}

\newpage

\section{ECL: Full Results}
\label{app:ecl_full}

We benchmark on ECL (321 hourly electricity consumption series, UCI, 2012--2014), the canonical long-term forecasting benchmark used by Autoformer, PatchTST, and iTransformer.

\textbf{Protocol.}
All models use a 96-hour context window (4 days), matching the standard ECL evaluation context. Prediction horizons: $H\in\{96, 192, 336, 720\}$ hours.
Data is normalized (zero-mean, unit-variance per series) before training and
evaluation; this is the canonical ``normalized'' protocol, producing metrics
directly comparable to published LTSF results.
In Table~\ref{tab:ecl_full}, TDA alone provides no gain or mild regression at all horizons. The sheaf drives consistent improvements for both adapter families from H96 through H336, with gains attenuating at H720 as the static topological coordinate becomes less marginal over the backbone's long-range distributional prior.

\begin{table}[H]
\centering
\caption{ECL (321 electricity customers): MAE / MSE at $H\in\{96,192,336,720\}$ hours.
 All trained variants: 96-hour context window. Normalized protocol.
 \textbf{Bold} = best per section. $\downarrow$ lower is better.}
\label{tab:ecl_full}
\small\setlength{\tabcolsep}{4pt}
\resizebox{\linewidth}{!}{%
\begin{tabular}{l rr rr rr rr}
\toprule
& \multicolumn{2}{c}{\textbf{H96}}
& \multicolumn{2}{c}{\textbf{H192}}
& \multicolumn{2}{c}{\textbf{H336}}
& \multicolumn{2}{c}{\textbf{H720}} \\
\cmidrule(lr){2-3}\cmidrule(lr){4-5}\cmidrule(lr){6-7}\cmidrule(lr){8-9}
\textbf{Model}
 & \textbf{MAE}$\downarrow$ & \textbf{MSE}$\downarrow$
 & \textbf{MAE}$\downarrow$ & \textbf{MSE}$\downarrow$
 & \textbf{MAE}$\downarrow$ & \textbf{MSE}$\downarrow$
 & \textbf{MAE}$\downarrow$ & \textbf{MSE}$\downarrow$ \\
\midrule
\multicolumn{9}{l}{\textit{Literature: published LTSF benchmarks}} \\
Autoformer~\citep{wu2021autoformer}
 & 0.317 & 0.201 & 0.334 & 0.222 & 0.338 & 0.231 & 0.361 & 0.254 \\
PatchTST~\citep{nie2023patchtst}
 & 0.285 & 0.195 & 0.289 & 0.199 & 0.305 & 0.215 & 0.337 & 0.256 \\
iTransformer~\citep{liu2024itransformer}
 & \best{0.270} & \best{0.178} & \best{0.274} & \best{0.182} & \best{0.292} & \best{0.200} & \best{0.320} & \best{0.220} \\
\midrule
\multicolumn{9}{l}{\textit{Transformer variants}} \\
Transformer
 & \best{0.193} & \best{0.091} & 0.234 & 0.125 & \best{0.243} & 0.140 & 0.355 & 0.289 \\
Transformer $+$ TDA
 & 0.197 & 0.102 & \best{0.227} & \best{0.119} & \best{0.243} & 0.141 & \best{0.351} & \best{0.276} \\
Transformer $+$ TDA $+$ Sheaf
 & 0.196 & \best{0.091} & 0.231 & \best{0.119} & 0.245 & \best{0.136} & 0.355 & 0.279 \\
\midrule
\multicolumn{9}{l}{\textit{Chronos 2.0 variants}} \\
Chronos Zero-Shot & 0.586 & 0.610 & 0.594 & 0.616 & 0.612 & 0.638 & 0.642 & 0.688 \\
Chronos Vanilla Adapter
 & 0.302 & 0.205 & 0.305 & 0.207 & 0.319 & 0.221 & 0.348 & 0.259 \\
Chronos $+$ TDA
 & 0.302 & 0.205 & 0.305 & 0.207 & 0.318 & 0.220 & 0.349 & 0.259 \\
Chronos $+$ TDA $+$ Sheaf
 & \best{0.290} & \best{0.190} & \best{0.298} & \best{0.198} & \best{0.315} & \best{0.217} & \best{0.346} & \best{0.257} \\
\midrule
\multicolumn{9}{l}{\textit{TimesFM 2.5 variants}} \\
TimesFM Zero-Shot
 & 0.580 & 0.602 & 0.581 & 0.599 & 0.589 & 0.605 & 0.607 & 0.626 \\
TimesFM Vanilla Adapter
 & 0.300 & 0.204 & 0.303 & 0.206 & 0.315 & 0.219 & \best{0.342} & \best{0.252} \\
TimesFM $+$ TDA
 & 0.300 & 0.204 & 0.303 & 0.206 & 0.315 & 0.219 & 0.343 & 0.254 \\
TimesFM $+$ TDA $+$ Sheaf
 & \best{0.289} & \best{0.190} & \best{0.296} & \best{0.197} & \best{0.313} & \best{0.217} & 0.350 & 0.269 \\
\bottomrule
\end{tabular}}
\end{table}

\newpage


\begin{thebibliography}{99}
\bibitem[Ansari et al.(2025)]{ansari2025chronos2}
A.~F.~Ansari et al.
\newblock {Chronos-2: From Univariate to Universal Forecasting}.
\newblock \emph{arXiv preprint arXiv:2510.15821}, 2025.
\bibitem[Bodnar et al.(2022)]{bodnar2022neural}
C.~Bodnar, F.~Di~Giovanni, B.~Chamberlain, P.~Li\'{o}, and M.~Bronstein.
\newblock {Neural Sheaf Diffusion}.
\newblock In \emph{NeurIPS}, 2022.
\bibitem[Bubenik(2015)]{bubenik2015statistical}
P.~Bubenik.
\newblock {Statistical Topological Data Analysis Using Persistence Landscapes}.
\newblock \emph{Journal of Machine Learning Research}, 16(1):77-102, 2015.
\bibitem[Carlsson(2009)]{carlsson2009topology}
G.~Carlsson.
\newblock {Topology and Data}.
\newblock \emph{Bulletin of the American Mathematical Society}, 46(2):255-308,
  2009.
\bibitem[Tralie et al.(2018)]{ctralie2018ripser}
C.~Tralie, N.~Saul, and R.~Bar-On.
\newblock {Ripser.py: A Lean Persistent Homology Library for Python}.
\newblock \emph{Journal of Open Source Software}, 3(29):925, 2018.
\bibitem[Curry(2014)]{curry2014sheaves}
J.~Curry.
\newblock \emph{Sheaves, Cosheaves and Applications}.
\newblock PhD thesis, University of Pennsylvania, 2014.
\bibitem[Das et al.(2024)]{das2024timesfm}
A.~Das et al.
\newblock {TimesFM}: A decoder-only foundation model for time-series
  forecasting.
\newblock In \emph{ICML}, 2024.
\bibitem[Edelsbrunner and Harer(2010)]{edelsbrunner2010computational}
H.~Edelsbrunner and J.~Harer.
\newblock \emph{Computational Topology: An Introduction}.
\newblock American Mathematical Society, 2010.
\bibitem[Hansen and Ghrist(2021)]{hansen2021sheaf}
J.~Hansen and R.~Ghrist.
\newblock Opinion dynamics on discourse sheaves.
\newblock \emph{SIAM Journal on Applied Mathematics}, 81(5):2033-2060, 2021.
\bibitem[Li et al.(2018)]{li2018dcrnn}
Y.~Li et al.
\newblock {Diffusion Convolutional Recurrent Neural Network}.
\newblock In \emph{ICLR}, 2018.
\bibitem[Lim et al.(2021)]{lim2021tft}
B.~Lim, S.~\"{O}.~Ar\i{}k, N.~Loeff, and T.~Pfister.
\newblock {Temporal Fusion Transformers}.
\newblock \emph{International Journal of Forecasting}, 37(4):1748-1764, 2021.
\bibitem[Lin et al.(2025a)]{lin2025crosstoponet}
Z.~Lin, N.~F.~S.~Zulkepli, M.~S.~M.~Kasihmuddin, and R.~U.~Gobithaasan.
\newblock {CrossTopoNet: A Cross-Attention Framework on Topological Latent
  Feature Space for Time-Series Forecasting}.
\newblock \emph{Knowledge-Based Systems}, 2025.
\bibitem[Lin et al.(2025b)]{lin2025tis}
Z.~Lin and N.~F.~S.~Zulkepli.
\newblock {Time-Series Forecasting via Topological Information Supervised
  Framework with Efficient Topological Feature Learning}.
\newblock \emph{arXiv preprint arXiv:2503.23757v1}, 2025.
\newblock \textit{Withdrawn by authors; cited for methodological comparison.}
\bibitem[Mostafa et al.(2026)]{mostafa2026stsheaf}
A.~Mostafa, R.~Younis, and Z.~Ahmadi.
\newblock {Dynamic Sheaf Diffusion Networks with Adaptive Local Structure for
  Heterogeneous Spatio-Temporal Graph Learning}.
\newblock \emph{arXiv preprint arXiv:2604.11275v1}, 2026.
\bibitem[Nie et al.(2023)]{nie2023patchtst}
Y.~Nie et al.
\newblock {A Time Series is Worth 64 Words}.
\newblock In \emph{ICLR}, 2023.
\bibitem[Papillon et al.(2024)]{papillon2024tdl}
M.~Papillon, S.~Sanborn, M.~Hajij, et al.
\newblock {Position: Topological Deep Learning is the New Frontier for
  Relational Learning}.
\newblock In \emph{ICML}, 2024.
\bibitem[Kim et al.(2025)]{kim2025topocl}
N.~Kim, H.~Baik, and Y.~Yoon.
\newblock {TopoCL: Topological Contrastive Learning for Time Series}.
\newblock \emph{arXiv preprint arXiv:2502.02924}, 2025.
\bibitem[Wang et al.(2019)]{wang2019deepfactors}
Y.~Wang, A.~Smola, D.~Maddix, J.~Gasthaus, D.~Foster, and T.~Januschowski.
\newblock {Deep Factors for Forecasting}.
\newblock In \emph{ICML}, 2019.
\bibitem[Chen and Guestrin(2016)]{chen2016xgboost}
T.~Chen and C.~Guestrin.
\newblock {XGBoost: A Scalable Tree Boosting System}.
\newblock In \emph{KDD}, 2016.
\bibitem[Wu et al.(2019)]{wu2019graphwavenet}
Z.~Wu et al.
\newblock {Graph WaveNet for Deep Spatial-Temporal Graph Modeling}.
\newblock In \emph{IJCAI}, 2019.
\bibitem[Wu et al.(2020)]{wu2020mtgnn}
Z.~Wu et al.
\newblock {Connecting the Dots: Multivariate Time Series Forecasting}.
\newblock In \emph{KDD}, 2020.
\bibitem[Zeng et al.(2023)]{zeng2023transformers}
A.~Zeng, M.~Chen, L.~Zhang, and Q.~Xu.
\newblock {Are Transformers Effective for Time Series Forecasting?}
\newblock In \emph{AAAI}, 2023.
\bibitem[Zeng et al.(2021)]{zeng2021topological}
S.~Zeng, F.~Graf, C.~Hofer, and R.~Kwitt.
\newblock {Topological Attention for Time Series Forecasting}.
\newblock In \emph{NeurIPS}, volume~34, 2021.
\bibitem[Wu et al.(2021)]{wu2021autoformer}
H.~Wu, J.~Xu, J.~Wang, and M.~Long.
\newblock {Autoformer: Decomposition Transformers with Auto-Correlation for
  Long-Term Series Forecasting}.
\newblock In \emph{NeurIPS}, volume~34, 2021.
\bibitem[Liu et al.(2024)]{liu2024itransformer}
Y.~Liu et al.
\newblock {iTransformer: Inverted Transformers Are Effective for Time Series
  Forecasting}.
\newblock In \emph{ICLR}, 2024.
\bibitem[Zhou et al.(2021)]{zhou2021informer}
H.~Zhou et al.
\newblock {Informer: Beyond Efficient Transformer for Long Sequence
  Forecasting}.
\newblock In \emph{AAAI}, 2021.

\end{thebibliography}
\end{document}